\documentclass[sigconf]{acmart}

\AtBeginDocument{%
  \providecommand\BibTeX{{%
    \normalfont B\kern-0.5em{\scshape i\kern-0.25em b}\kern-0.8em\TeX}}}

\copyrightyear{2022}
\acmYear{2022}
\setcopyright{acmcopyright}\acmConference[MM '22]{Proceedings of the 30th ACM International Conference on Multimedia}{October 10--14, 2022}{Lisboa, Portugal}
\acmBooktitle{Proceedings of the 30th ACM International Conference on Multimedia (MM '22), October 10--14, 2022, Lisboa, Portugal}
\acmPrice{15.00}
\acmDOI{10.1145/3503161.3548334}
\acmISBN{978-1-4503-9203-7/22/10}
\usepackage{enumerate}
\usepackage{bm}
\usepackage{amsthm}  
\usepackage{amsmath}
\usepackage{algorithm}
\usepackage{algpseudocode}
\usepackage{multirow}
\usepackage{siunitx}
\usepackage{graphicx,subfig,color,multirow,amsmath}
\usepackage{caption}
\usepackage{algorithmicx}

\begin{document}

\title{Multiple Kernel Clustering with Dual Noise Minimization}



\author{Junpu Zhang}
\authornote{Both authors contributed equally to this research.}
\email{zhangjunpu@nudt.edu.cn}
\author{Liang Li}
\authornotemark[1]
\email{liangli@nudt.edu.cn}
\affiliation{%
 \institution{National University of Defense Technology}
 \city{Changsha}
 \country{China}
}

\author{Siwei Wang}
\email{wangsiwei13@nudt.edu.cn}
\affiliation{%
 \institution{National University of Defense Technology}
 \city{Changsha}
 \country{China}
}

\author{Jiyuan Liu}
\email{liujiyuan13@nudt.edu.cn}
\affiliation{%
 \institution{National University of Defense Technology}
 \city{Changsha}
 \country{China}
}

\author{Yue Liu}
\email{yueliu@nudt.edu.cn}
\affiliation{%
 \institution{National University of Defense Technology}
 \city{Changsha}
 \country{China}
}

\author{Xinwang Liu}
\authornote{Corresponding author.}
\email{xinwangliu@nudt.edu.cn}
\affiliation{%
 \institution{National University of Defense Technology}
 \city{Changsha}
 \country{China}
}

\author{En Zhu}
\email{enzhu@nudt.edu.cn}
\authornotemark[2]
\affiliation{%
 \institution{National University of Defense Technology}
 \city{Changsha}
 \country{China}
}

\renewcommand{\shortauthors}{Junpu Zhang and Liang Li, et al.}

\begin{abstract}
Clustering is a representative unsupervised method widely applied in multi-modal and multi-view scenarios. Multiple kernel clustering (MKC) aims to group data by integrating complementary information from base kernels. As a representative, late fusion MKC first decomposes the kernels into orthogonal partition matrices, then learns a consensus one from them, achieving promising performance recently. However, these methods fail to consider the noise inside the partition matrix, preventing further improvement of clustering performance. We discover that the noise can be disassembled into separable dual parts, i.e. N-noise and C-noise (Null space noise and Column space noise). In this paper, we rigorously define dual noise and propose a novel parameter-free MKC algorithm by minimizing them. To solve the resultant optimization problem, we design an efficient two-step iterative strategy. To our best knowledge, it is the first time to investigate dual noise within the partition in the kernel space. We observe that dual noise will pollute the block diagonal structures and incur the degeneration of clustering performance, and C-noise exhibits stronger destruction than N-noise. Owing to our efficient mechanism to minimize dual noise, the proposed algorithm surpasses the recent methods by large margins.
\end{abstract}



\maketitle

\section{Introduction}
Clustering is a representative unsupervised learning method widely applied in data mining, community detection and many other machine learning scenarios \cite{nie2014clustering,ren2020self,nie2016constrained,liu2022simple,ren2020deep,zhang2017latent,xu2021multi}. Multi-view or multi-modal clustering aims to optimally fuse diverse and complementary information, which has been a hotpot in current research \cite{chang2015convex,zhan2017graph,wen2018incomplete,zhan2018graph,kang2019robust,huang2020auto,huang2021cdd,9646486,wang2022align}. As Figure \ref{MVC} shows, how to effectively and efficiently integrate multimedia or multiple features, e.g. image, video, and text, is still an open question \cite{jing2013dictionary,nie2017multi,jing2006ontology,liu2022efficient,zhang2020deep,liu2022improved,wen2021cdimc,wang2022highly}. 
Multiple kernel clustering (MKC) \cite{zhao2009multiple,zhang2010general,kang2017kernel,huang2019auto,kang2019low,ren2021multiple} is a popular technique to solve this. Considering the insufficiency to tackle nonlinearly separable data in sample space, MKC maps the sample space to a Reproducing Kernel Hilbert Space (RKHS), where the data can be linearly separable \cite{shawe2004kernel}. Currently, there are two mainstream methods, including kernel fusion and late fusion strategies.

\begin{figure}[t]
\begin{center}{
		\centering
		\includegraphics[width=0.475\textwidth]{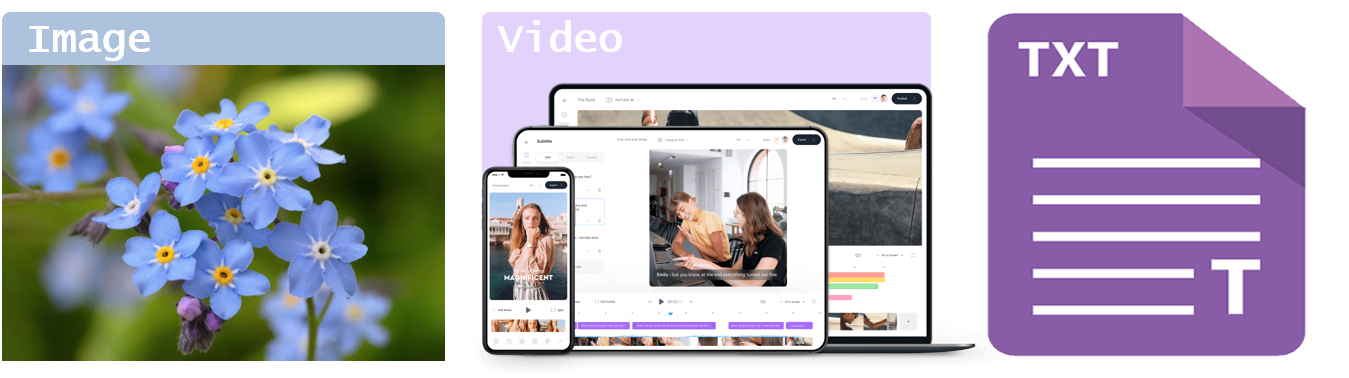}
		\vspace{-5pt}
		\caption{Multi-view learning aims to fuse data across multimedia or multiple features, e.g. features can be extract from image, video, text and other multi-modal representations.}
		\label{MVC}}
\end{center}
\vspace{-15pt}
\end{figure}
\begin{figure}[!t]
\begin{center}{
		\centering
            \subfloat[\scriptsize Original kernel (ACC: 41.77\%)] {\includegraphics[width=0.2\textwidth] {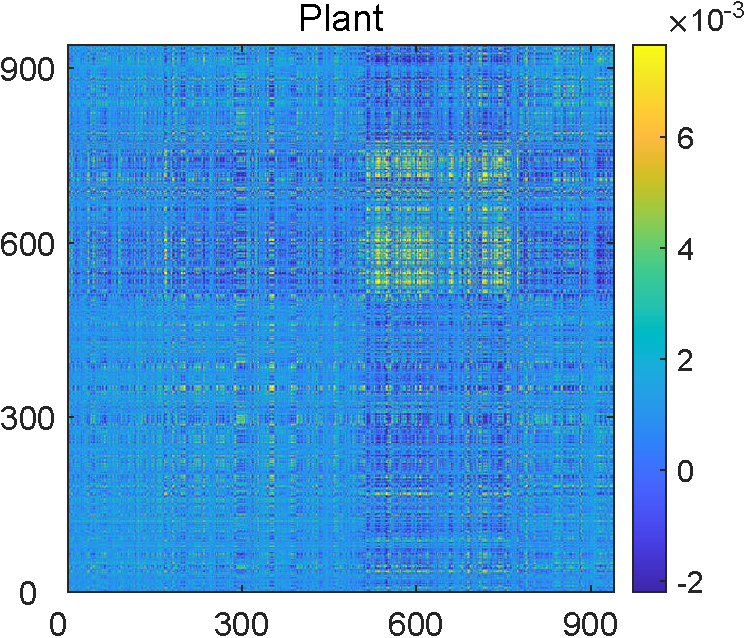}} \hspace{5mm}
            \subfloat[\scriptsize Remove N-noise (ACC: 68.29\%)] {{\includegraphics[width=0.2\textwidth]{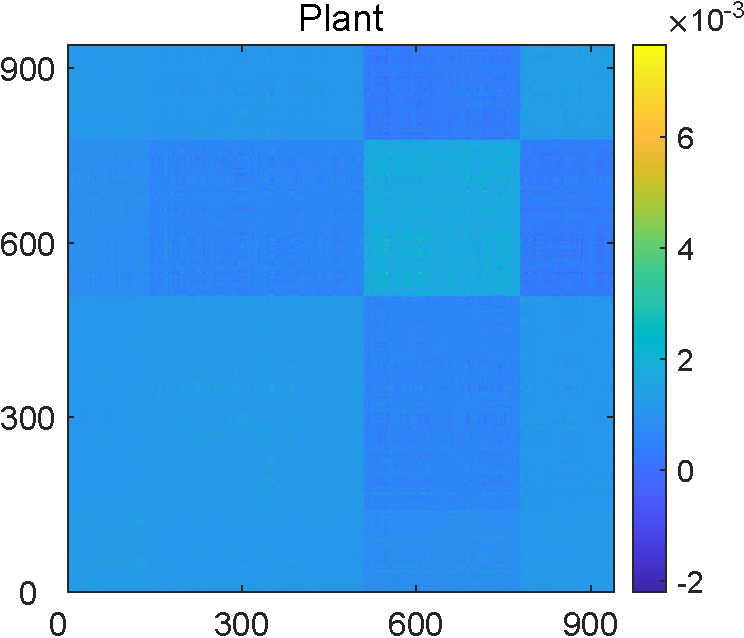}}}\\
    		\vspace{-5pt}
            \subfloat[\scriptsize {Remove C-noise (ACC: 97.62\%)}] {{\includegraphics[width=0.2\textwidth]{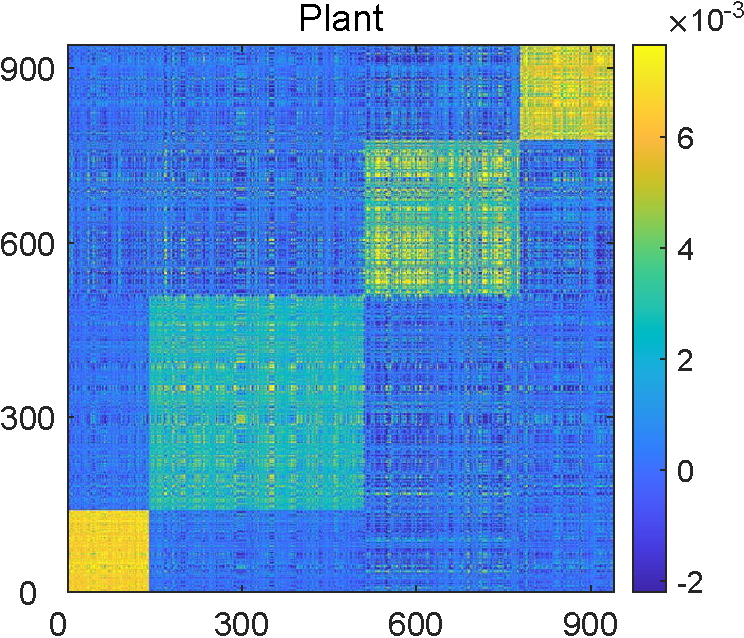}}}\hspace{5mm}
            \subfloat[\scriptsize Remove dual noise (ACC: 100\%)] {{\includegraphics[width=0.2\textwidth]{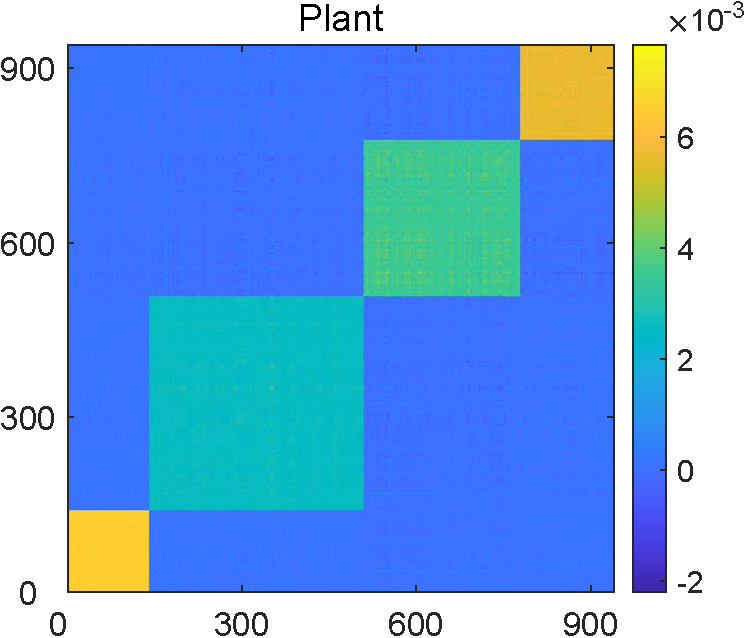}}}
    		\vspace{-5pt}
			\caption{Visualization of noise destruction on Plant dataset.}
            \label{Visualization_plant}
}
\end{center}
\vspace{-15pt}
\end{figure}
Kernel fusion methods focus on learning a consensus kernel from base kernels directly, afterwards compute the final partition (cluster soft-assignment) \cite{dhillon2004kernel}. 
A typical paradigm is multiple kernel $k$-means (MKKM) \cite{huang2012multiple}.
Meanwhile, plenty of variants are derived \cite{liu2016multiple,kang2017twin,liu2017optimal,kang2019similarity,liang2022LSWMKC}. For the purpose of directly serving for clustering tasks, late fusion methods aim to obtain a consensus partition from base partitions. This strategy is proposed by \cite{wang2019multi} and inspires a large number of researches \cite{liu2018late,liu2021hierarchical,liu2021one,wang2021late,9653838}. Our proposed algorithm belongs to the second category.

Although the late fusion methods exhibit promising performance, most existing researches \cite{wang2019multi,liu2021one,wang2021late,9653838} encounter three issues: (i) These models adopt a coarse manner that directly fuses the pre-computed kernel partitions without updating during the optimization. 
Consequently, the quality of consensus partition is greatly limited by the initial partitions and leads to limited clustering performance. 
The work in \cite{liu2021hierarchical} attempts to tackle this issue by updating partitions in a hierarchical manner. 
However, along with the improvement of clustering performance, it introduces a great complexity in the optimization. 
(ii) Moreover, most late fusion based methods are modeled with one or more hyper-parameters, which is intractable in real-world scenarios due to the missing of supervisory signals. 
(iii) Most critically, existing researches fail to consider the noise within partition matrices.
In clustering settings, researchers always prefer a clear block diagonal structure. 
However, as Figure \ref{Visualization_plant} (a) shows, the noise will inevitably corrupt the block diagonal structure, leading to the degeneration of clustering performance. 
Overall, an efficient, parameter-free model which can effectively minimize the impact of noise is an urgent need in multiple kernel clustering applications.

To fill these gaps, this paper develops a novel MKC algorithm with a dual noise minimization mechanism (MKC-DNM). Specifically, we discover that the noise, according to its mathematical property, can be disassembled into two separate dual parts, i.e. Null space noise (N-noise) and Column space noise (C-noise). As Figure \ref{Visualization_plant} shows, we visualize the effect of dual noise on Plant dataset. 
Specifically, we test the kernel quality (the accuracy of kernel $k$-means) in four comparative settings, i.e. without modification, removing N-noise, removing C-noise and removing both of them.
It can be observed that the accuracy increases from 41.77\% to 68.29\% and 97.62\% with removing N-noise and C-noise, respectively.
The phenomenon illustrates (i) both of the dual noise will pollute the kernel, leading to the degeneration of clustering performance; (ii) C-noise exhibits stronger destruction than N-noise on the block diagonal structure. 
Therefore, a natural motivation of this work is to minimize dual noise.
This paper firstly provides rigorous mathematical definitions of dual noise, then carefully explores their properties, and further proposes a unified and elegant paradigm to minimize them simultaneously. 
The contributions of this work are summarized as follows: 
\begin{enumerate}[1)]
\item In MKC scenarios, for the first time, we mathematically disassemble the noise of kernel partition into N-noise and C-noise, distinguishing our work from existing researches. Furthermore, we find that C-noise exhibits stronger destruction than N-noise on the block diagonal structures, which directly leads to the degeneration of clustering performance.
\item We propose a novel model to minimize dual noise in late fusion framework. Most importantly, our model is parameter-free, making it practical, especially in unsupervised scenarios.
\item We propose an efficient two-step alternative optimization strategy to solve our model with linear computation complexity, and achieve state-of-art clustering performance on benchmark datasets.
\end{enumerate}


\section{Related Work} \label{Se_Related_Work}
\subsection{Multiple Kernel $k$-means}
Considering a data matrix ${\mathbf{X} \in \mathbb{R}^{d \times n}}$ drawn from $k$ clusters where ${d}$ and $n$ refer to the feature dimension and sample number respectively, $k$-means aims to minimize the inter-cluster loss \cite{jing2007entropy,nie2011spectral,wen2020adaptive,zhang2018generalized},
\begin{equation}\label{KKM}
\begin{split}
\min_{\mathbf{Y}} \; \sum_{i=1}^{n}\sum_{r=1}^{k}\|\mathbf{x}_{i}-\mathbf{c}_{r}\|_{2}^{2}\mathbf{Y}_{ir}, \;\; \textrm{s.t.} \; \sum_{r=1}^{k} \mathbf{Y}_{ir}=1,
\end{split}
\end{equation}
where $\mathbf{Y}\in\{0,1\}^{n \times k}$ is the indicator matrix, $n_{r}=\sum_{i=1}^{n} \mathbf{Y}_{ir}$ is the sample number of the $r$-th cluster whose centroid is $\mathbf{c}_{r}$.

With kernel trick \cite{shawe2004kernel}, i.e. $\mathbf{K}_{ij}={\kappa}\left(\mathbf{x}_{i}, \mathbf{x}_{j}\right)=\phi\left(\mathbf{x}_{i}\right)^{\top} \phi\left(\mathbf{x}_{j}\right)$, the sample space $\mathbb{R}^{d}$ can be mapped into an RKHS $\mathcal{H}$ \cite{tzortzis2008global}, in which ${\kappa} \left(\cdot,\cdot\right)$ is the kernel function, ${\phi(\cdot)}$ is nonlinear feature mapping. Kernel $k$-means (KKM) is transformed to
\begin{equation}\label{KKM-matrix2}
\begin{split}
\min_{\mathbf{H}} \; \mathrm{Tr}\left(\left(\mathbf{I}-\mathbf{H} \mathbf{H}^{\top}\right)\mathbf{K}\right), \;\; \textrm {s.t.} \; \mathbf{H}^{\top} \mathbf{H}=\mathbf{I}_{k}, \;\mathbf{H} \in \mathbb{R}^{n \times k},
\end{split}
\end{equation}
where the partition matrix $\mathbf{H}$ is computed by eigenvalue decomposition.
The final cluster labels can be obtained by performing $k$-means on $\mathbf{H}$ \cite{ren2021multiple}.

In multiple kernel scenarios, the consensus kernel is commonly assumed as a combination of $m$ base kernels.
As a representative, the objective of MKKM \cite{huang2012multiple} is 
\begin{equation}\label{MKKM}
\begin{split}
\min _{\mathbf{H}, \boldsymbol{\beta}}& \;\; \operatorname{Tr}\left(\left(\mathbf{I}-\mathbf{H H}^{\top}\right)\mathbf{K}_{\boldsymbol{\beta}}\right), \\ 
\textrm {s.t.}&  \;\; \mathbf{H}^{\top} \mathbf{H}= \mathbf{I}_{k}, \;\mathbf{H} \in \mathbb{R}^{n \times k}, \; \boldsymbol{\beta} ^\top \mathbf{1} = 1, \;\beta_{p} \geq 0, \;\forall p,
\end{split}
\end{equation}
where $\mathbf{K}_{\boldsymbol{\beta}} = \sum_{p=1}^{m} \beta_{p}^{2} \mathbf{K}_{p}$ is the consensus kernel and $\beta_{p}$ is the weight of the $p$-th kernel. In the optimization, $\mathbf{H}$ and $\boldsymbol{\beta}$ can be solved alternatively.
\subsection{Late Fusion Multiple Kernel Clustering}
Instead of fusing consensus kernel from base kernels $\mathbf{K}_{p}$, late fusion MKC focuses on fusing multiple partitions $\mathbf{H}_{p}$ to directly serve for clustering \cite{wang2019multi}. The paradigm of late fusion MKC (LFMKC) is presented as follows:
\begin{equation}\label{Late_MKC}
\begin{split}
\max_{\mathbf{H^{\ast}},\mathbf{H_{\boldsymbol{\beta}}}, \boldsymbol{\beta}} \;\;& \operatorname{Tr}\left(\mathbf{H}^{\ast\top} \mathbf{H}_{\boldsymbol{\beta}} + \lambda \mathbf{H}^{\ast\top} \mathbf{M} \right), \\ 
\textrm {s.t.} \;\;\;& \mathbf{H}^{\ast\top} \mathbf{H}^{\ast}=\mathbf{I}_{k},\;\mathbf{H}^{\ast} \in \mathbb{R}^{n \times k},\\
&\mathbf{W}_{p}^{\top} \mathbf{W}_{p}=\mathbf{I}_{k},\;\mathbf{W}_{p} \in \mathbb{R}^{k \times k},\\
&\|\boldsymbol{\beta}\|^{2}_{2} = 1, \beta_{p} \geq 0, \forall p,
\end{split}
\end{equation}
where $\mathbf{H}^{\ast}$ is the consensus partition, $\mathbf{H}_{\boldsymbol{\beta}} = \sum_{p=1}^{m} \beta_{p} \mathbf{H}_{p} \mathbf{W}_{p}$ is the fused partition with each base one aligned by permutation matrix $\mathbf{W} _p$, $\mathbf{M}$ is the partition of average kernel, and $\lambda$ is a trade-off parameter.

The above paradigm aims to maximally align the consensus partition and base partitions. Although achieving promising performance, the consensus partition, as pointed in \cite{liu2021hierarchical}, is directly learned from base partitions that are fixed during the optimization, limiting its performance. 
Moreover, the current method neglects the noise in kernel partitions.

\subsection{Hierarchical Multiple Kernel Clustering}
To update the kernel partition during optimization, \cite{liu2021hierarchical} proposes to gradually category clusters from $\mathbf{K}_{p} \in \mathbb{R}^{n \times n}$ to intermediate $\mathbf{H} \in \mathbb{R}^{n \times c}$ and finally to $\mathbf{H}^{\ast} \in \mathbb{R}^{n \times k}$. The idea is formulated as 
\begin{equation}\label{HMKC}
\begin{split}
\max_{\mathbf{H^{\ast}},\mathbf{H}_{p}, \boldsymbol{\omega}, \boldsymbol{\beta}} \;\; &\sum_{t=1}^{s} \sum_{p=1}^{m} \omega_{p}^{(t)} \operatorname{Tr} \left( \mathbf{K}_{p}^{(t)} \mathbf{K}_{p}^{(t-1)} \right) + \sum_{p=1}^{m} \beta_{p} \operatorname{Tr}\left( \mathbf{K}_{p}^{(s)} \mathbf{K}^{\ast} \right),\\
\textrm {s.t.} \;\;\; &\mathbf{H}^{\ast\top} \mathbf{H}^{\ast} = \mathbf{I}_{k}, \; \mathbf{H}_{p}^{(t)\top} \mathbf{H}_{p}^{(t)}=\mathbf{I}_{c_t}, \; \mathbf{H}^{\ast} \in \mathbb{R}^{n \times k}, \\ &\mathbf{H}_{p}^{(t)} \in \mathbb{R}^{n \times c_t}, \; n > c_1 > \dots > c_s > k, \\
&\|\boldsymbol{\omega}^{(t)}\|^{2}_{2} = 1, \; \omega_{p}^{(t)} \geq 0, \; \|\boldsymbol{\beta}\|^{2}_{2} = 1, \; \beta_{p} \geq 0, 
\end{split}
\end{equation}
where $\mathbf{K}_{p}^{(t)} = \mathbf{H}_{p}^{(t)} {\mathbf{H}_{p}^{(t)\top}}$ for $t \geq 1$, $\mathbf{K}_{p}^{(0)} = \mathbf{K}_{p}$, $\mathbf{K}^{\ast} = \mathbf{H}^{\ast} \mathbf{H}^{\ast \top}$, and $\mathbf{H}_{p}^{(t)}$ is the intermediary partitions with decreasing sizes. The complex formulation conveys a straightforward insight that the data representation should be extracted step by step.
In this way, the data information beneficial to clustering could be maximally preserved.

Obviously, each term of Eq. (\ref{HMKC}) is a kernel $k$-means objective in essence. 
Also, this is an empirical and coarse manner.
Nevertheless, the sizes of intermediary partitions are still hyper-parameters that require carefully tuning or grid-searching in practice. 
Most importantly, it fails to tackle the noise inside the partition matrices.

\begin{figure*}[t]
\begin{center}{
		\centering
        {{\includegraphics[width=0.95\textwidth]{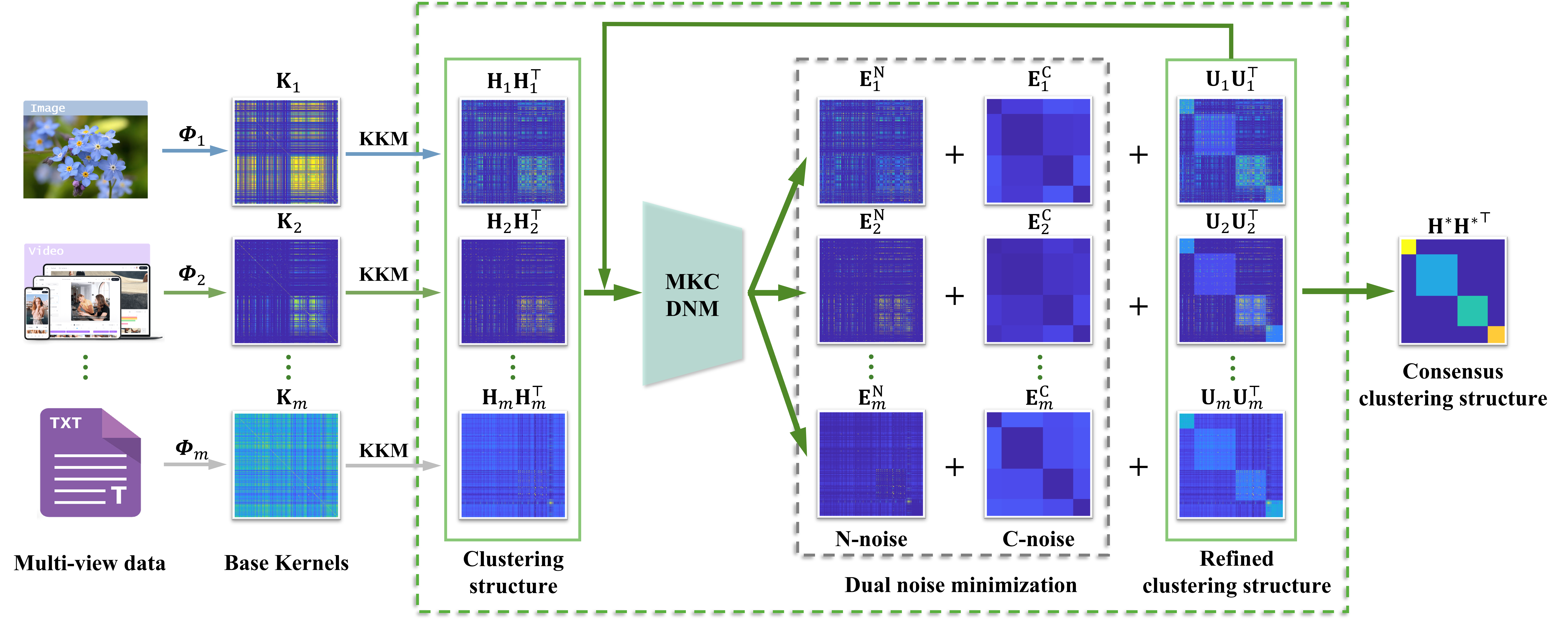}}}
		\caption{The framework of the proposed MKC-DNM model. The core idea is to adaptively optimize base partitions by minimizing dual noise during the iteration. Specifically, the MKC-DNM model firstly performs kernel $k$-means and generate $\mathbf{H}_{p}\mathbf{H}_{p}^{\top}$ to recover the clustering structures of base kernels $\mathbf{K}_{p}$, and then gradually remove N-noise ($\mathbf{E}_{p}^\mathrm{N}$) and C-noise ($\mathbf{E}_{p}^\mathrm{C}$) of $\mathbf{H}_{p}$, preserving the denoised feature matrices $\mathbf{U}_{p}$. Consequently, a consensus partition $\mathbf{H}^{\ast}$ with precise block diagonal structures is obtained.
		} 
		\label{framework}}
\end{center}
\end{figure*}
\section{Methodology}\label{Se_Methodology}
\subsection{Motivation}
In multiple kernel scenarios, a ${d_p}$-dimensional feature matrix $\mathbf{U}_p \in \mathbb{R}^{n \times d_{p}}$ is commonly served as the data representation of the $p$-th kernel computed by singular value decomposition (SVD), which satisfies $\mathbf{U}_p ^\top \mathbf{U}_p = \mathbf{I}_{d_p}$. 
It's worth noting that the dimension $d_{p}$ of feature matrices may vary at a large range since multiple kernels are naturally discrepant and complementary. Consequently, it is necessary to fuse a consensus optimal partition ${\mathbf{H}^\ast} \in \mathbb{R}^{n \times k}$ across multiple $\mathbf{U}_p$ for clustering purpose, i.e. 
\begin{equation}\label{eq:f}
\begin{split}
f\left(\{\mathbf{U}_p\}_{p=1}^m\right)=\mathbf{H}^\ast,\;\;\textrm {s.t.} \;{\mathbf{H}^\ast} ^\top \mathbf{H}^\ast = \mathbf{I}_{k},
\end{split}
\end{equation} 
where $f$ is the function to fuse $m$ feature matrices.
\subsubsection{Definitions of dual noise}
Directly integrating feature matrices $\mathbf{U}_p$ across $m$ kernels is a challenging issue since their dimensions $d_{p}$ varies greatly. 
Fortunately, both $\mathbf{U}_p \mathbf{U}_p^\top$ and $\mathbf{H}^\ast {\mathbf{H}^\ast}^\top$ share the target clustering structure of $\mathbf{K}_{p}$ but with a discrepancy $\mathbf{E}_p$, i.e. 
\begin{equation}\label{eq:E}
\begin{split}
\mathbf{U}_p \mathbf{U}_p^\top = \mathbf{H}^\ast {\mathbf{H}^\ast}^\top + \mathbf{E}_p,\;\; \forall p \in \Delta_m,
\end{split}
\end{equation}
where $\Delta_m = \{1,2,\cdots,m\}$, and $\mathbf{E}_p \in \mathbb{R}^{n \times n}$ can be regarded as the noise matrix. Mathematically, $\mathbf{E}_p$ can be further separated into Null space noise (N-noise, $\mathbf{E}_p ^\mathrm{N}$) and Column space noise (C-noise, $\mathbf{E}_p ^\mathrm{N}$), i.e.  
\begin{equation}\label{eq:ENEC}
\begin{split}
\mathbf{E}_p = \mathbf{E}_p ^\mathrm{N} + \mathbf{E}_p ^\mathrm{C},\;\; \forall p \in \Delta_m.
\end{split}
\end{equation}

We emphasize that the above definitions are derived based on which subspace their eigenvectors belong to, i.e. $v\left(\mathbf{E}_p^\mathrm{N}\right) \in N\left({\mathbf{H}^\ast}^\top\right)$, $v\left(\mathbf{E}_p ^\mathrm{C}\right) \in C\left(\mathbf{H}^\ast\right)$, where $v\left(\mathbf{A}\right)$ denotes the eigenvectors with corresponding non-zero eigenvalues of matrix $\mathbf{A}$, $N\left(\mathbf{B}\right)$ and $C\left(\mathbf{B}\right)$ denote the Null space and the Column space of matrix $\mathbf{B}$, respectively. Mathematically, as pointed in \cite{greub2012linear}, $v\left(\mathbf{E}_p^\mathrm{N}\right)$, $v\left(\mathbf{E}_p ^\mathrm{C}\right)$ can be computed by     
\begin{align}
&v\left(\mathbf{E}_p ^\mathrm{N}\right) \in N\left({\mathbf{H}^\ast}^\top\right) = \left\{ \mathbf{x} \;|\; {\mathbf{H}^\ast}^\top \mathbf{x} = \mathbf{0}\right\},\;\; \forall p \in \Delta_m, \label{eq:EN}\\
&v\left(\mathbf{E}_p ^\mathrm{C}\right) \in C\left(\mathbf{H}^\ast\right) = \left\{ \mathbf{y} \;|\; \exists \; \mathbf{x}, \; \mathrm{s.t.}\; \mathbf{y} = \mathbf{H}^\ast \mathbf{x}\right\},\;\; \forall p \in \Delta_m. \label{eq:EC}
\end{align}

According to \cite{greub2012linear}, the null space of ${\mathbf{H}^\ast}^\top$ is the orthogonal complement of the column space of $\mathbf{H}^\ast$, which demonstrates that given $\mathbf{E}_p$, dual noise matrices $\mathbf{E}_p ^\mathrm{N}$ and $\mathbf{E}_p ^\mathrm{C}$ exist and should be unique. 
Figure \ref{Visualization_plant} gives a visualization of destruction caused by dual noise in kernel space. 
\subsubsection{Properties of dual noise}
Before introducing the proposed noise minimization mechanism, we first give several vital mathematical properties of $\mathbf{E}_p ^\mathrm{N}$ and $\mathbf{E}_p ^\mathrm{C}$ in Lemma \ref{lemma1}-\ref{lemma3}.

\begin{lemma}\label{lemma1}
$\mathrm{Tr} \left( \mathbf{E}_p ^\mathrm{N} \mathbf{H}^\ast {\mathbf{H}^\ast}^\top \right) = 0$,\;\; $\forall p \in \Delta_m$.
\end{lemma}

\begin{lemma}\label{lemma2}
$\mathrm{Tr} \left( \mathbf{E}_p ^\mathrm{C} \mathbf{H}^\ast {\mathbf{H}^\ast} ^\top \right) = \mathrm{Tr}\left( \mathbf{E}_p ^\mathrm{C}\right)$,\;\; $\forall p \in \Delta_m$.
\end{lemma}

\begin{lemma}\label{lemma3}
$\mathbf{E}_p ^\mathrm{N}$ is positive semi-definite (PSD) and $\mathbf{E}_p ^\mathrm{C}$ is negative semi-definite (NSD).
\end{lemma}

Lemma \ref{lemma1}-\ref{lemma2} illustrate the relationship between the dual noise and the optimal partition $\mathbf{H}^\ast$, and Lemma \ref{lemma3} shows that we can remove dual noise by minimizing the absolute value of their trace. 
Due to space limit, the detailed proofs of Lemma \ref{lemma1}-\ref{lemma3} are provided in the appendix.

\subsubsection{Minimizing C-noise} \label{Minimizing-C}
Recall our motivation that we aim to minimize C-noise, ideally $\mathbf{E}_p ^\mathrm{C} = 0$, according to Lemma \ref{lemma2} and Lemma \ref{lemma3}, it is equivalent to
\begin{equation}\label{eq:influ_EC}
\begin{split}
\mathrm{Tr}\left( \mathbf{E}_p ^\mathrm{C} \mathbf{H}^\ast {\mathbf{H}^\ast}^\top \right) = 0,\;\; \forall p \in \Delta_m.
\end{split}
\end{equation}
However, directly solving Eq. (\ref{eq:influ_EC}) is difficult due to the unknown $\mathbf{H}^\ast$. 
Fortunately, Theorem \ref{thm1} provides a necessary condition to satisfy $\mathbf{E}_p ^\mathrm{C} = 0$. 

\newtheorem{thm}{\bf Theorem}[section]
\begin{thm}\label{thm1}
$\|\mathbf{U}_p^\top \mathbf{U}_q\|_\mathrm{F}^2 \geq k$ is necessary for $\mathbf{E}_p ^\mathrm{C} = 0$, $\forall p \in \Delta_m$. 
\end{thm} 

\begin{proof}
Given $\mathbf{E}_p ^\mathrm{C} = 0$ and $\mathbf{H}^{\ast\top} {\mathbf{H}^\ast} = \mathbf{I}_{k}$, according to Eq. (\ref{eq:f}), we have
\begin{equation}\label{eq:proof1.1}
\begin{split}
&\|\mathbf{U}_p^\top \mathbf{U}_q\|_\mathrm{F}^2 = \mathrm{Tr}\left(\mathbf{U}_p \mathbf{U}_p^\top \mathbf{U}_q \mathbf{U}_q^\top\right) \\
=& \mathrm{Tr}\left( \left(\mathbf{H}^\ast {\mathbf{H}^\ast}^\top + \mathbf{E}_p ^\mathrm{N}\right) \left(\mathbf{H}^\ast {\mathbf{H}^\ast}^\top + \mathbf{E}_q ^\mathrm{N}\right) \right)\\
=& \mathrm{Tr}\left( \mathbf{H}^\ast {\mathbf{H}^\ast}^\top \mathbf{H}^\ast {\mathbf{H}^\ast}^\top + \mathbf{E}_p ^\mathrm{N} \mathbf{H}^\ast {\mathbf{H}^\ast}^\top  + \mathbf{H}^\ast {\mathbf{H}^\ast}^\top \mathbf{E}_q ^\mathrm{N} + \mathbf{E}_p ^\mathrm{N} \mathbf{E}_q ^\mathrm{N} \right).
\end{split}
\end{equation}

According to Lemma (\ref{lemma1}), Eq. (\ref{eq:proof1.1}) is equivalent to 
\begin{equation}\label{eq:proof1.2}
\begin{split}
&\mathrm{Tr}\left( \mathbf{H}^\ast {\mathbf{H}^\ast}^\top \mathbf{H}^\ast {\mathbf{H}^\ast}^\top + \mathbf{E}_p ^\mathrm{N} \mathbf{E}_q ^\mathrm{N} \right) = k + \mathrm{Tr}\left( \mathbf{E}_p ^\mathrm{N} \mathbf{E}_q ^\mathrm{N} \right). 
\end{split}
\end{equation}

Since $\mathbf{E}_p ^\mathrm{N}$ satisfies PSD according to Lemma \ref{lemma3}, we have $\mathrm{Tr}\left( \mathbf{E}_p ^\mathrm{N} \mathbf{E}_q ^\mathrm{N} \right) \geq 0$, i.e. Eq. (\ref{eq:proof1.2}) $\geq k$. Consequently, $\|\mathbf{U}_p^\top \mathbf{U}_q\|_\mathrm{F}^2 \geq k $. 

This completes the proof.
\end{proof}
\subsubsection{Minimizing N-noise} \label{Minimizing-N}
Similarly, the optimal solution of minimizing N-noise is $\mathbf{E}_p^ \mathrm{N} = \mathbf{0}$. 
According to Lemma \ref{lemma3}, it is equivalent to $\mathrm{Tr}(\mathbf{E}_p^ \mathrm{N}) = 0$. Therefore, the original goal to minimize $\mathbf{E}_p^ \mathrm{N}$ can be transformed to minimize $\mathrm{Tr}(\mathbf{E}_p^ \mathrm{N})$. For MKC scenario, we should minimize $\sum_{p=1}^{m} \mathrm{Tr}(\mathbf{E}_p^ \mathrm{N})$. Furthermore, Theorem \ref{thm2} illustrates that minimizing $\sum_{p=1}^{m} \mathrm{Tr}(\mathbf{E}_p^ \mathrm{N})$ is equivalent to minimizing $\mathbf{d}^\top \mathbf{1}$. 

\begin{thm}\label{thm2}
Fixing $\{\mathbf{E}_p^\mathrm{C}\}_{p=1}^m$, minimizing $\sum_{p=1}^{m} \mathrm{Tr}(\mathbf{E}_p^ \mathrm{N})$ is equivalent to minimizing $\mathbf{d}^\top \mathbf{1}$. 
\end{thm} 

\begin{proof}
Given $\mathbf{U}_p \mathbf{U}_p^\top = \mathbf{I}_{d_p}$ and $\mathbf{H}^{\ast\top} {\mathbf{H}^\ast} = \mathbf{I}_{k}$ for all $p \in \Delta_m$, according to Eq. (\ref{eq:E}), we have
\begin{equation}\label{eq:proof2.1}
\begin{split}
&\sum_{p=1}^{m} \mathrm{Tr}\left(\mathbf{E}_p^ \mathrm{N}\right) = \sum_{p=1}^{m}\left(\mathrm{Tr} \left( \mathbf{U}_p \mathbf{U}_p^\top\right) - \mathrm{Tr}\left(\mathbf{H}^\ast {\mathbf{H}^\ast}^\top + \mathbf{E} _p^{C} \right)\right) \\
=& \sum_{p=1}^{m} {d_p} - \sum_{p=1}^{m} \left({k} + \mathrm{Tr}\left(\mathbf{E}_p^{C}\right)\right) \iff \sum_{p=1}^{m} {d_p} = \mathbf{d} ^\top \mathbf{1}, 
\end{split}
\end{equation}
where $\mathbf{d} = [d_1, d_2, \cdots, d_m]^{\top}$ collects the dimension of feature matrices $\mathbf{U}_p$. Note that the above deduction is valid since $\sum_{p=1}^{m} \left({k} + \mathrm{Tr}\left(\mathbf{E}_p^{C}\right)\right)$ is a constant and can be eliminated for optimization.  

This completes the proof.
\end{proof}
\subsection{Proposed Formulation}
According to the aforementioned analysis on minimizing dual noise, we integrate Theorem \ref{thm1} and Theorem \ref{thm2} into a unified and parameter-free framework, i,e.   
\begin{equation}\label{eq:model1}
\begin{split}
\min_\mathbf{d} \; \mathbf{d}^\top \mathbf{1}, \;\;
\mathrm{s.t.}\; \|\mathbf{U}_p^\top \mathbf{U}_q\|_\mathrm{F}^2 \geq k,\; \forall p,q \in \Delta_m;\; \mathbf{d} \in \mathbb{Z}_+^m,
\end{split}
\end{equation}
where $\mathbf{d} = [d_1, d_2, \cdots, d_m]^{\top}$ collects the dimensions of feature matrices $\mathbf{U}_p$, and $\mathbb{Z} _+^m$ denotes the set of all positive integers.

From the above compact model, we have the following observations: (i) Our motivation aims to minimize dual noise, and we derive a straightforward but elegant framework. (ii) The objective is derived from minimizing N-noise, and the constraint is originated from minimizing C-noise, which directly serves for noise minimization purpose. (iii) Our model is parameter-free, satisfying the essence of unsupervised clustering.

Although our algorithm is compact and elegant, it is difficult to solve Eq. (\ref{eq:model1}) directly. Inspired by the widely employed Big M method \cite{griva2009linear} in operation research, we introduce auxiliary variables $\{ a_p \}_{p=1} ^m$ to transform the original Eq. (\ref{eq:model1}) into the following formulation:
\begin{equation}\label{eq:model2}
\begin{split}
\min_{\mathbf{d},\hat{\mathbf{d}}} \;\;& \frac{1}{2}(\mathbf{d} + \hat{\mathbf{d}})^\top \mathbf{1} + \frac{\mathrm{M}}{2}\|\mathbf{d} - \hat{\mathbf{d}}\|_2^2,\\
\mathrm{s.t.}\;\; &\|\mathbf{U}_p^\top \hat{\mathbf{U}}_q\|_\mathrm{F}^2 \geq k,\; \forall p,q \in \Delta_m;\; \mathbf{d}, \hat{\mathbf{d}} \in \mathbb{Z}_+^m,
\end{split}
\end{equation}
where $\hat{\mathbf{d}} = \mathbf{d} + \mathbf{a}$ and $\hat{\mathbf{U}}_q \in \mathbb{R}^{n \times \hat{d}_p}$. As pointed in \cite{griva2009linear}, the solution of Eq. (\ref{eq:model1}) is equivalent to Eq. (\ref{eq:model2}) for a large enough $\mathrm{M}$, i.e. $\|\mathbf{d} - \hat{\mathbf{d}} \|_2^2$ will be zero. Theorem \ref{thm3} gives a brief proof.   
\begin{thm}\label{thm3}
The local optimal solution of Eq. (\ref{eq:model2}) is equivalent to that of Eq. (\ref{eq:model1}) for a large enough constant $\mathrm{M}$.
\end{thm} 

\begin{proof}
We first define two functions: 
\begin{equation}\label{}
\begin{split}
g(\mathbf{d}) = \mathbf{d} ^\top \mathbf{1},\;\; G_\mathrm{M}(\mathbf{d}, \hat{\mathbf{d}}) = \frac{1}{2}(\mathbf{d} + \hat{\mathbf{d}})^\top \mathbf{1} + \frac{\mathrm{M}}{2}\|\mathbf{d} - \hat{\mathbf{d}}\|_2^2, 
\end{split}
\end{equation}
Supposing $\left( \mathbf{d} ^\ast, \hat{\mathbf{d}} ^\ast \right)$ is a local optimal solution of Eq. (\ref{eq:model2}), which can be formulated as
\begin{equation}\label{neib1}
\begin{split}
\exists \; N \left( \mathbf{d} ^\ast, \hat{\mathbf{d}} ^\ast \right),\; \forall \left( \mathbf{d}, \hat{\mathbf{d}} \right) \in N\left( \mathbf{d} ^\ast, \hat{\mathbf{d}} ^\ast \right),\; G_\mathrm{M} \left( \mathbf{d} ^\ast, \hat{\mathbf{d}} ^\ast \right) \leq G_\mathrm{M}\left( \mathbf{d}, \hat{\mathbf{d}} \right), 
\end{split}
\end{equation}
where $N\left( \bullet \right)$ denotes the neighborhood of $\bullet$.

A large enough $\mathrm{M}$ 
means $\|\mathbf{d} ^\ast - \hat{\mathbf{d}} ^\ast \|_2^2 = 0$, i.e. $\mathbf{d} ^\ast = \hat{\mathbf{d}} ^\ast$. According to Eq. (\ref{neib1}), we have
\begin{equation}\label{neib2}
\begin{split}
\exists N\left( \mathbf{d} ^\ast \right),\; \forall \mathbf{d} \in N\left( \mathbf{d}^\ast \right),\; G_\mathrm{M} \left( \mathbf{d} ^\ast, \mathbf{d} ^\ast \right) \leq G_\mathrm{M}\left( \mathbf{d}, \mathbf{d} \right). 
\end{split}
\end{equation}

Furthermore, we have 
\begin{equation}\label{}
\begin{split}
g\left( \mathbf{d} ^\ast \right) = G_\mathrm{M} \left( \mathbf{d} ^\ast, \mathbf{d} ^\ast \right) \leq G_\mathrm{M} \left( \mathbf{d}, \mathbf{d}  \right) = g\left( \mathbf{d} \right).
\end{split}
\end{equation}

This completes the proof. 
\end{proof}

\subsection{Optimization}
Directly optimizing Eq. (\ref{eq:model2}) is difficult since it is not convex. This section provides an alternative strategy where each sub-optimization is convex.
\begin{algorithm}[tb]
  \caption{Multiple kernel clustering with dual noise minimization}
  \label{alg:optimization}
  \begin{algorithmic}[1]
    \Require
     $\{{\mathbf{H}}_p\} _{p=1} ^m$ and $k$. 
    \Ensure
      The optimal $\{{\mathbf{U}}_p \} _{p=1} ^m$ and $\mathbf{d}$.
    \State Initialize a feasible $\mathbf{d}$, $\hat{\mathbf{d}} = \mathbf{d}$, flag $= 1$;
    \While{flag}
      \State update $\mathbf{d}$ by optimizing Eq. (\ref{eq:model2_dp});
      \State update $\hat{\mathbf{d}}$ by optimizing Eq. (\ref{eq:model2_dp2});
    \If{$\mathbf{d} = \hat{\mathbf{d}}$}
      \State flag $= 0$.
    \EndIf
    \EndWhile
  \end{algorithmic}
\end{algorithm}

\subsubsection{Sub-optimization of updating $\bm{d}$}\label{section:solve_d}

With fixed $\hat{\mathbf{d}}$, Eq. (\ref{eq:model2}) is formulated into
\begin{equation}\label{eq:model2_d}
\begin{split}
\min_{\{d_p\}_{p=1}^m} &\sum_{p=1}^m \frac{1}{2} (d_p + \hat{d}_p) + \frac{\mathrm{M}}{2} (d_p - \hat{d}_p)^2,\\
\mathrm{s.t.}\;\; &\|\mathbf{U}_p^\top \hat{\mathbf{U}}_q\|_\mathrm{F}^2 \geq k,\;\; \forall q \in \Delta_m;\; d_p \in \mathbb{Z}_+.
\end{split}
\end{equation}

As pointed in \cite{DBLP:books/daglib/0009359}, the object function of Eq. (\ref{eq:model2_d}) is $\mathrm{L} ^\natural$-convex on the effective domain, and the effective domain is restricted to an $\mathrm{L} ^\natural$-convex set. Consequently, Eq. (\ref{eq:model2_d}) is an $\mathrm{L} ^\natural$-convex problem about $\mathbf{d}$. 

Since each $d_p$ is independent, Eq. (\ref{eq:model2_d}) can be transformed into $m$ independent sub-problems and each can be separately solved by
\begin{equation}\label{eq:model2_dp}
\begin{split}
\min_{d_p} \;\;&\frac{1}{2} (d_p + \hat{d}_p) + \frac{\mathrm{M}}{2} (d_p - \hat{d}_p)^2,\\
\mathrm{s.t.}\;\; &\|\mathbf{U}_p^\top \hat{\mathbf{U}}_q\|_\mathrm{F}^2 \geq k,\; \forall q \in \Delta_m,\; d_p \in \mathbb{Z}_+.
\end{split}
\end{equation}

According to \cite{DBLP:series/natosec/Onn11}, the above one-dimensional $\mathrm{L} ^\natural$-convex optimization problem can be solved by discrete binary search efficiently. 
\subsubsection{Sub-optimization of updating $\hat{\bm{d}}$}
Similarly, we can obtain $\hat{\mathbf{d}}$ by solving the following problem separately for each $p$: 
\begin{equation}\label{eq:model2_dp2}
\begin{split}
\min_{\hat{d}_p} \;\;&\frac{1}{2} (\hat{d}_p + d_p) + \frac{\mathrm{M}}{2} (\hat{d}_p - d_p)^2,\\
\mathrm{s.t.}\;\; &\|{\hat{\mathbf{U}}_p}^\top \mathbf{U}_q\|_\mathrm{F}^2 \geq k,\; \forall q \in \Delta_m,\; \hat{d}_p \in \mathbb{Z}_+.
\end{split}
\end{equation}

The optimization procedures in solving Eq. (\ref{eq:model2}) is outlined in Algorithm \ref{alg:optimization}. Note that the optimal $\mathbf{U}_p \mathbf{U}_p ^\top$ learned by Algorithm \ref{alg:optimization} is PSD for all $p \in \Delta_m$, which can be regarded as kernel matrix. As a result, we can employ average kernel $k$-means on $\{\mathbf{U}_p \mathbf{U}_p ^\top\} _{p=1}^m$ to recover the optimal $\mathbf{H} ^\ast$, which can be solved efficiently by performing SVD on $\left[ \mathbf{U}_1, \mathbf{U}_2, \cdots, \mathbf{U}_m \right]$ and 
extract the rank-$k$ columns of the left singular matrix \cite{kang2020large}. 

\subsection{Complexity and Convergence}

\subsubsection{Computational Complexity}\label{Complexity}
The computation complexity of our algorithm involves two part, i.e. optimization and post-processing. In optimization, it involves two variables. Updating $\mathbf{d}$ involves $m$ independent sub-problems to compute the optimal $\mathbf{d}$, and each one conducts discrete binary search. Therefore, the computational complexity is $\mathcal{O} \left( n m^2 k^2 \log k \right)$. Similarly, updating $\hat{\mathbf{d}}$ also needs $\mathcal{O} \left( n m^2 k^2 \log k \right)$ for solving Eq. (\ref{eq:model2_dp2}). Therefore, the optimization process costs $\mathcal{O} \left( n m^2 k^2 \log k \right)$. Note that it achieves linear complexity with respect to sample number. 
For post-processing, we perform SVD on $\left[ \mathbf{U}_1, \mathbf{U}_2, \cdots, \mathbf{U}_m \right]$ to obtain the optimal partition matrix $\mathbf{H}^{\ast}$ and get the final clustering label by $k$-means. The complexity is $\mathcal{O} \left( n m^2 k^2 \right)$, which is also linear complexity with respect to $n$. Consequently, we develop a linear-complexity algorithm with $\mathcal{O} \left( n m^2 k^2 \log k \right)$, sharing the similar computational complexity with late fusion methods \cite{wang2019multi} and \cite{liu2021one}. Note that linear complexity means it is suitable for large-scale tasks. Moreover, our algorithm is free of hyper-parameter, which is more practical compared with the ones requiring parameter-tuning by grid search.
\subsubsection{Convergence}
Our algorithm is non-convex to directly compute all the variables, and we adopt an alternative optimization manner to solve the resultant model. According to Theorem 1 in \cite{Bezdek2003}, alternatively optimizing each sub-optimization is convex with other variables fixed. The objective of Algorithm \ref{alg:optimization} is monotonically decreased and lower bounded by zero. Consequently, our proposed algorithm can be guaranteed convergent to a local optimal solution.


\section{Experiment}\label{Se_Experiment}
\begin{table}[tbp]
\caption{Information of datasets}
\label{Datasets}
\centering
 \resizebox{0.42\textwidth}{!}{
  \small
\begin{tabular}{ccccc}\toprule
Dataset         & Type  & Samples & Views   & Clusters \\\midrule
BBCSport        & Text  & 544   & 2  & 5   \\
Plant           & Image & 940   & 69 & 4   \\
SensIT Vehicle  & Graph & 1500  & 2  & 3   \\
CCV             & Video & 6773  & 3  & 20  \\
Flower102       & Image & 8189  & 4  & 102 \\
Reuters         & Text  & 18758 & 5  & 6  \\\bottomrule
\end{tabular}}
\vspace{-10pt}
\end{table}
\begin{table*}[t] 
  \caption{Comparing the ACC, NMI, purity and ARI of eleven algorithms on six MVC datasets. The best one is marked in bold, the second best is marked in italic underline. `N/A' denotes out of memory and time-out errors. `Average Rank' records the average rank across six datasets.}
  \label{Performance}
\begin{center}
\vspace{-10pt}
\small   
{
  \centering
  \resizebox{\textwidth}{!}{
  \small
    \begin{tabular}{cccccccccccc}
    \toprule
    \multirow{2}{*}{Datasets}  &\multirow{2}{*}{AKKM} &{MKKM}  &{MKKM-MR} &{SWMC}  &{ONKC}  &{LFMKC} &{SPMKC} &{CAGL} &{OPLF} &{HMKC}  &\multirow{2}{*}{Proposed}\\
                               &                         &\cite{huang2012multiple} &\cite{liu2016multiple}   &\cite{nie2017self} &\cite{liu2017optimal} &\cite{wang2019multi}  &\cite{ren2020simultaneous} &\cite{ren2020consensus} &\cite{liu2021one} &\cite{liu2021hierarchical} & \\
    \midrule
    Number of parameter
    & - & - & 1 & - & 2 & 1 & 2 & 2 & - & 2 & -\\
    \midrule
    \multicolumn{12}{c} {ACC $(\%)$} \\
    \midrule
BBCSport & 39.47 & 39.38 & 39.51 & 36.76 & 39.71 & 60.06 & 46.51 & 71.51 &  60.85 & {\underline{\textit{71.90}} } & {\textbf{87.61}} \\
Plant & \underline{\textit{61.28}} & 56.05 & 50.27 & 38.94 & 41.43 & 59.53 & 32.87 & 43.09 & 48.51 & 59.21 & {\textbf{61.79}} \\
SensIT Vehicle & 53.73 & 53.36 & 54.13 & 34.67 & 54.21 & 66.28 & 54.20 & 34.40 & 54.87 & \underline{\textit{66.60}} & {\textbf{72.53}} \\
CCV & 19.63 & 17.99 & 21.24 & 10.84 & 22.39 & 25.13 & 9.67 & 12.58 & 22.74 & \underline{\textit{33.31}} & {\textbf{35.04}} \\
Flower102 & 27.07 & 22.41 & 40.22 & 6.72 & 39.55 & 38.45 & N/A & 26.25 & 29.78 & {\textbf{48.21}} & \underline{\textit{47.93}} \\
Reuters & 45.46 & 45.44 & 46.15 & 27.12 & 41.85 & 45.68 & N/A & N/A & 44.65 & \underline{\textit{46.84}} & {\textbf{48.17}}\\
Average Rank &6.33 & 7.83 & 5.50 & 10.00 & 6.33 & 3.83 & 8.50  & 7.80  & 5.33 & 2.17 & \textbf{1.17} \\
    \midrule
    \multicolumn{12}{c} {NMI $(\%)$} \\
    \midrule
BBCSport & 15.74 & 15.69 & 15.77 & 2.63 & 16.10 & 40.38 & 23.89 & {\textbf{72.74}} & 41.46 & 50.50 & \underline{\textit{69.70}} \\
Plant & 26.53 & 19.49 & 20.37 & 0.51 & 10.49 & 23.35 & 0.21 & 11.90 & 13.67 & \underline{\textit{27.98}} & {\textbf{31.06}} \\
SensIT Vehicle & 10.83 & 10.25 & 11.32 & 1.55 & 11.31 & \underline{\textit{23.53}} & 20.31 & 1.45 & 12.28 & 22.08 & {\textbf{32.09}} \\
CCV & 16.84 & 15.04 & 18.03 & 1.07 & 18.52 & 20.09 & 1.60 & 6.08 & 18.72 & \underline{\textit{29.85}} & {\textbf{30.79}} \\
Flower102 & 46.02 & 42.67 & 56.71 & 5.51 & 56.11 & 54.94 & N/A & 45.09 & 46.77 & {\textbf{61.92}} & \underline{\textit{61.73}} \\
Reuters & 27.37 & 27.35 & 25.30 & 1.35 & 22.27 & 27.39 & N/A & N/A & 27.09 & {\textbf{31.04}} & \underline{\textit{30.64}}\\
Average Rank & 6.33 & 7.83 & 5.83 & 10.33 & 6.67 & 3.67 & 7.75 & 7.40 & 5.33 & 2.00 & \textbf{1.50} \\
    \midrule
    \multicolumn{12}{c} {Purity $(\%)$} \\
    \midrule
BBCSport & 48.89 & 48.86 & 48.91 & 37.50 & 49.13 & 68.78 & 52.39 & 73.90 & 68.57 & \underline{\textit{74.65}} & {\textbf{87.61}} \\
Plant & \underline{\textit{61.28}} & 56.05 & 56.71 & 39.36 & 49.02 & 59.53 & 39.15 & 46.81 & 52.87 & 59.77 & {\textbf{62.10}} \\
SensIT Vehicle & 53.73 & 53.36 & 54.13 & 35.13 & 54.21 & 66.28 & 54.20 & 34.40 & 54.87 & \underline{\textit{66.60}} & {\textbf{72.53}} \\
CCV & 23.75 & 22.18 & 23.74 & 11.35 & 24.64 & 28.16 & 11.78 & 13.58 & 26.52 & \underline{\textit{37.05}} & {\textbf{38.32}} \\
Flower102 & 32.27 & 27.79 & 46.34 & 8.08 & 45.63 & 44.56 & N/A & 31.33 & 34.03 & \underline{\textit{54.81}} & {\textbf{54.93}} \\
Reuters & 53.01 & 52.94 & 52.15 & 28.25 & 52.63 & 53.23 & N/A & N/A & 52.92 & \underline{\textit{53.91}} & {\textbf{55.37}}\\
Average Rank & 6.00 & 7.83 & 6.33 & 10.17 & 6.00 & 3.67 & 8.25 & 8.00 & 5.33 & 2.17 & \textbf{1.00} \\
    \midrule
    \multicolumn{12}{c} {ARI $(\%)$} \\
    \midrule
BBCSport & 9.28 & 9.25 & 9.31 & 0.47 & 9.63 & 31.57 & 14.30 & \underline{\textit{56.64}} & 30.90 & 44.94 & {\textbf{69.60}} \\
Plant & \underline{\textit{24.64}} & 17.42 & 19.03 & -0.25 & 9.78 & 21.66 & -0.39 & 1.76 & 12.28 & 24.41 & {\textbf{28.90}} \\
SensIT Vehicle & 10.84 & 10.29 & 11.33 & 0.09 & 11.36 & 25.32 & 13.99 & 0.03 & 13.57 & \underline{\textit{25.55}} & {\textbf{35.51}} \\
CCV & 6.60 & 5.75 & 7.15 & -0.02 & 7.74 & 9.44 & -0.01 & 0.29 & 8.02 & \underline{\textit{14.74}} & {\textbf{15.70}} \\
Flower102 & 15.46 & 12.06 & 25.49 & 0.10 & 24.86 & 25.46 & N/A & 2.28 & 18.92 & {\textbf{34.34}} & \underline{\textit{34.23}} \\
Reuters & 21.84 & 21.82 & {\textbf{23.08}} & 0.09 & 20.32 & 22.10 & N/A & N/A & 21.25 & 22.59 & \underline{\textit{22.72}}\\
Average Rank & 6.33 & 7.83 & 5.00 & 10.33 & 6.50 & 3.67 & 7.75 & 8.00 & 5.67 & 2.33 & \textbf{1.33} \\
    \bottomrule
    \end{tabular}}}
    \vspace{-5pt}
\end{center}
\end{table*} 
\subsection{Datasets}
Table \ref{Datasets} lists six MKC datasets, including {BBCSport \footnote{\footnotesize{\texttt{http://mlg.ucd.ie/datasets/bbc.html}}}},
{Plant\footnote{\footnotesize{\texttt{https://bmi.inf.ethz.ch/supplements/protsubloc}}}},
{SensIT Vehicle\footnote{\footnotesize{\texttt{https://www.csie.ntu.edu.tw/~cjlin/libsvmtools/datasets/multiclass.html}}}},
{CCV\footnote{\footnotesize{\texttt{https://www.ee.columbia.edu/ln/dvmm/CCV/}}}},
{Flower102\footnote{\footnotesize{\texttt{http://www.robots.ox.ac.uk/\~{}vgg/data/flowers/}}}}, and
{Reuters\footnote{\footnotesize{\texttt{http://kdd.ics.uci.edu/databases/reuters21578/reuters21578.html}}}}.
The datasets vary in type and size, which will provide convincing evaluation for this work. 
All the datasets are downloaded form public websites.  
\begin{table*}[t] 
\caption{Comparing the CPU time. `Tune Time' denotes execution time including hyper-parameter tuning. `Run Time' denotes running time for this algorithm. `-' denotes a parameter-free algorithm. `N/A' denotes out of memory and time-out errors.} 
\label{Time_record}
\begin{center}
\vspace{-5pt}
\small
{
\centering
\resizebox{\textwidth}{!}{
\begin{tabular}{ccccccccccccc}
    \toprule
Datasets &  & AKKM & MKKM & MKKM-MR & SWMC & ONKC & LFMKC & SPMKC & CAGL & OPLF & HMKC & Proposed \\
    \midrule
\multirow{2}{*}{BBCSport} & Tune Time & - & - & 2.17 & - & 155.13 & 0.74 & 28.88 & 185.68 & 1.78 & 27.04 & - \\
 & Run Time & 0.04 & 0.04 & 0.07 & 7.17 & 0.16 & 0.02 & 0.80 & 1.15 & 0.09 & 0.38 & 0.02 \\
\multirow{2}{*}{Plant} & Tune Time & - & - & 1032.42 & - & 31740.88 & 25.19 & 1119.20 & 9813.97 & 21.04 & 1542.74 & - \\
 & Run Time & 0.16 & 3.88 & 33.30 & 101.68 & 33.03 & 0.81 & 31.09 & 60.58 & 1.05 & 21.43 & 15.29 \\
\multirow{2}{*}{SensIT Vehicle} & Tune Time & - & - & 16.63 & - & 1734.24 & 3.28 & 441.30 & 1532.34 & 3.41 & 127.52 & - \\
 & Run Time & 0.04 & 0.56 & 0.54 & 57.23 & 1.80 & 0.11 & 12.26 & 9.46 & 0.17 & 1.77 & 0.22 \\
\multirow{2}{*}{CCV} & Tune Time & - & - & 1598.44 & - & 347868.59 & 208.92 & 364927.26 & 181525.46 & 148.10 & 21775.50 & - \\
 & Run Time & 1.64 & 42.67 & 51.56 & 5512.67 & 361.99 & 6.74 & 10136.87 & 1120.53 & 7.40 & 302.44 & 55.99 \\
\multirow{2}{*}{Flower102} & Tune Time & - & - & 4343.86 & - & 624586.09 & 1195.65 & N/A & 382874.57 & 973.52 & 395255.46 & - \\
 & Run Time & 7.27 & 109.76 & 140.12 & 7612.89 & 649.93 & 38.57 & N/A & 2363.42 & 48.68 & 5489.66 & 237.29 \\
\multirow{2}{*}{Reuters} & Tune Time & - & - & 49510.69 & - & 3743171.77 & 1571.38 & N/A & N/A & 983.08 & 93400.45 & - \\
 & Run Time & 11.43 & 830.46 & 1597.12 & 84059.26 & 3895.08 & 50.69 & N/A & N/A & 49.15 & 1297.23 & 44.57\\
    \bottomrule 
\end{tabular}}}
\end{center}
\end{table*}
\subsection{Compared Methods}
Ten existing graph or kernel based multi-view clustering (MVC) models are adopted as baseline, i.e. \\
(1) {\bf AKKM} conducts KKM in average kernel space.\\   
(2) {\bf {MKKM}} \cite{huang2012multiple} combines base kernels with learned weights.\\
(3) {\bf {MKKM-MR}} \cite{liu2016multiple} proposes a matrix-induced regularization to avoid the sparsity of weight.\\
(4) {\bf {SWMC}} \cite{nie2017self} is a self-weighted method.\\  
(5) {\bf {ONKC}} \cite{liu2017optimal} learns the optimal neighborhood kernel.\\
(6) {\bf {LFMKC}} \cite{wang2019multi} maximizes the alignment of permuted base partitions and the consensus one.\\  
(7) {\bf {SPMKC}} \cite{ren2020simultaneous} simultaneously extract the global and local clustering structures by graph learning in MVC.\\
(8) {\bf {CAGL}} \cite{ren2020consensus} fuses a consensus graph across multiple kernels by graph learning with rank constraint.\\
(9) {\bf {OPLF}} \cite{liu2021one} is a one-pass version of LFMKC.\\
(10) {\bf {HKMC}} \cite{liu2021hierarchical} reduces the dimension of data representation hierarchically rather than abruptly.\\

Note that (i) the kernel fusion MKC algorithms, including AKKM, MKKM, MKKM-MR, and ONKC, (ii) graph learning based method SWMC, (iii) graph based MKC models, i.e. SPMKC and CAGL, share the similar computational complexity, i.e. $\mathcal{O}(n^{3})$. (iv) late fusion based algorithms, i.e. LFMKC and OPLF, are of $\mathcal{O}(n)$ complexity. Since HKMC involves eigenvalue decomposition during the optimization, it costs $\mathcal{O}(n^{3})$ complexity. 
\subsection{Experimental Setup}
For the MKC datasets, we suppose the real clusters $k$ is pre-knowledge. All the public datasets are centered and normalized by following \cite{shawe2004kernel}. For the methods involving $k$-means, the cluster centroids are initialized 50 times randomly and we report the average results, reducing the randomness. 
For the compared methods, the involved hyper-parameters of their public codes are carefully tuned according to authors' settings. For the proposed MKC-DNM model, the involved parameter $\mathrm{M}$ in the optimization is determined in the following way to accelerate the procedure: $\mathrm{M}$ is firstly set with a small value, then in each iteration, we multiply $\mathrm{M}$ by two until satisfying the convergence condition, i.e. $\mathbf{d} = \hat{\mathbf{d}}$. The clustering performance is analyzed by four metrics, i.e. Accuracy (ACC), Normalized Mutual Information (NMI), Purity, and Adjusted Rand Index (ARI) \cite{huang2018robust,wen2019unified,ren2019parallel,wen2020generalized,zhang2015low}. The experiments are conducted on a computer with Intel Core i9 10900X CPU (3.5GHz), 64GB RAM, and MATLAB 2020b (64bit).
\subsection{Clustering Performance}
Table \ref{Performance} lists four metrics comparison. The best results are marked in bold, the second best ones are marked in italic underline, and `N/A’ means unavailable results due to out of memory or time-out errors. Moreover, we record the average rank for a reference. From the reported results, we have the following observations:
\begin{enumerate}[1)]
\item The proposed algorithm exhibits the best or second best performance on six datasets. Compared to the late fusion based MKC methods, i.e. LFMKC, OPLF, and HMKC methods, our algorithm exhibits obvious improvement on most datasets, owing to our noise minimization mechanism.
\item Especially, HMKC can be considered as the most powerful competitor in MKC, our model still achieves comparable and better solutions on most datasets, with an obvious margin of 15.71\%, 2.59\%, 5.93\%, 5.32\%, 1.73\%, and 1.33\% of ACC on six datasets. We emphasize that our model is parameter-free in contrast to HMKC which introduces two empirical hyper-parameters and reports the best solution.
\item According to the average rank, our algorithm achieves the first rank over HMKC, LFMKC and OPLF methods, illustrating the superiority of our MKC-DNM algorithm.    
\end{enumerate}
\vspace{-10pt}

\begin{figure*}[!t]
\vspace{-5pt}
\begin{center}{
		\centering
            \hspace{-2pt}		
            \subfloat[ACC]{{\includegraphics[width=0.165\textwidth]{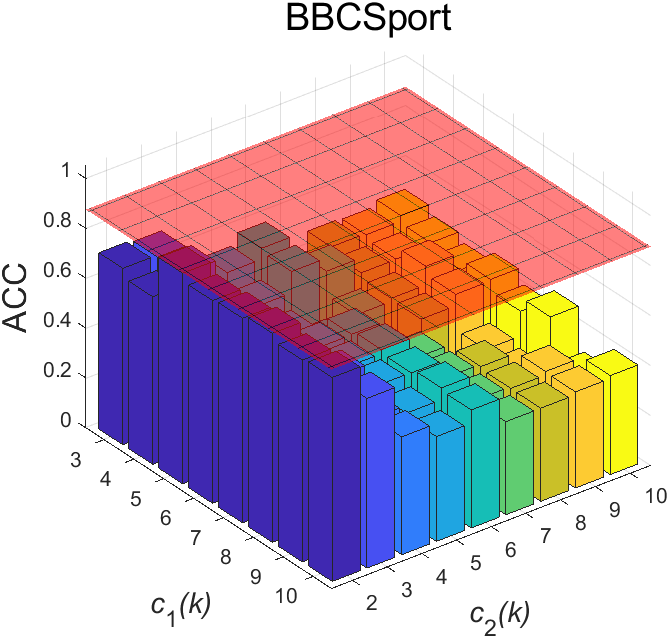}}}\hspace{-2pt}
            \subfloat[NMI]{{\includegraphics[width=0.165\textwidth]{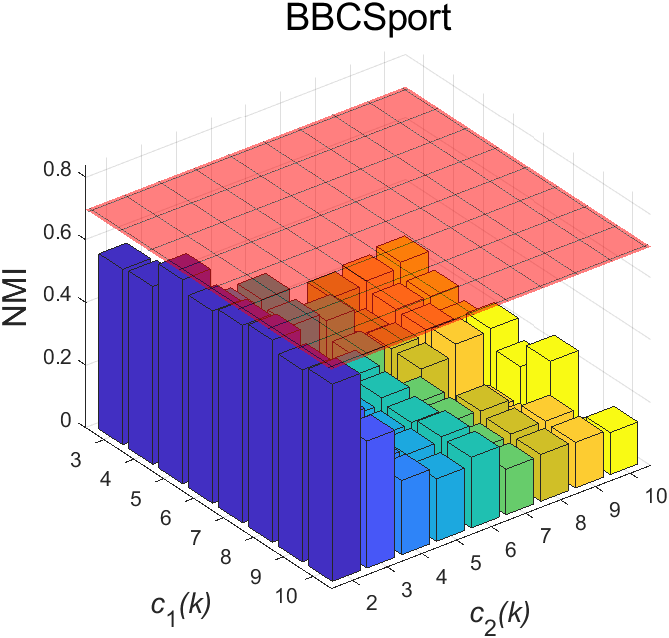}}}\hspace{-2pt}
            \subfloat[ACC]{{\includegraphics[width=0.165\textwidth]{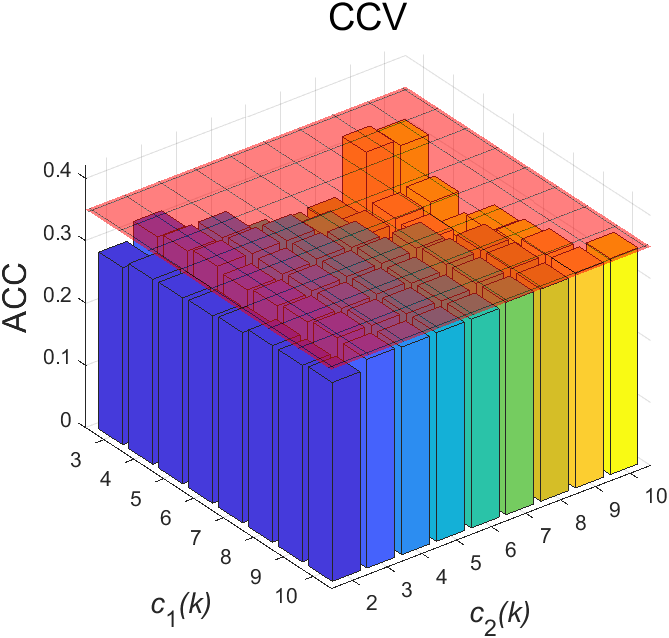}}}\hspace{2pt}
            \subfloat[NMI]{{\includegraphics[width=0.165\textwidth]{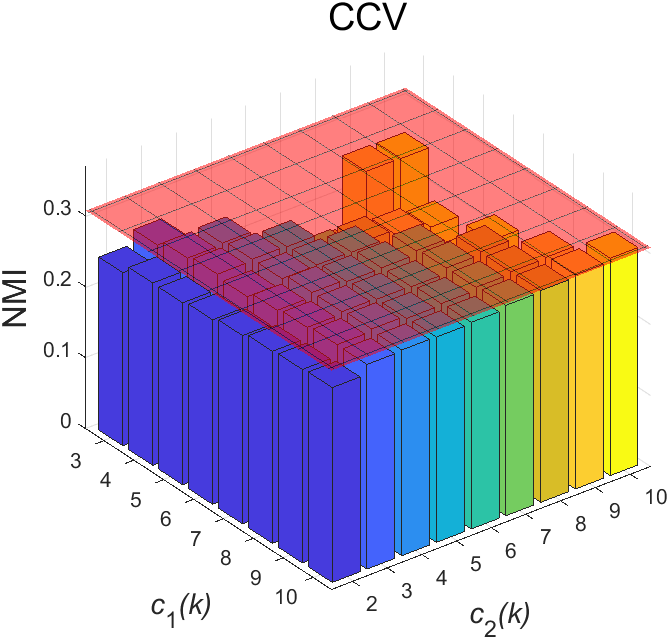}}}\hspace{-2pt}
            \subfloat[ACC]{{\includegraphics[width=0.165\textwidth]{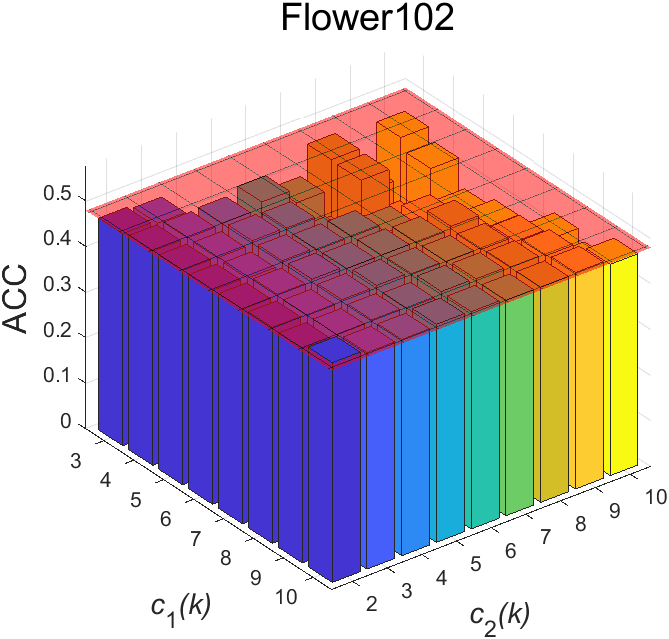}}}\hspace{-2pt}
            \subfloat[NMI]{{\includegraphics[width=0.165\textwidth]{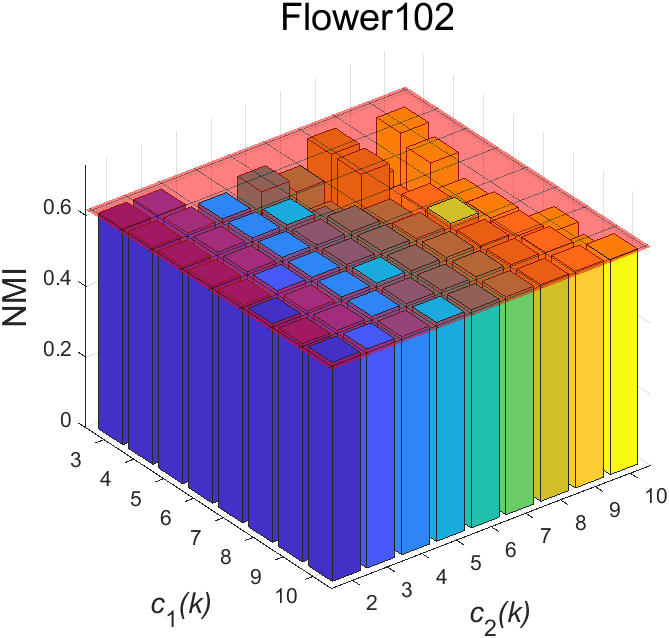}}}
            \hspace{-2pt}
			\caption{{Comparing the strongest competitor HMKC and our algorithm on BBCSport, CCV and Flower102. The red slice denotes the performance of our parameter-free MKC-DNM. The 3-D bar denotes the results of HKMC which involves two hyper-parameters tuning by grid search.}}
			\label{Compare_HKMC}
			}
\end{center}
\vspace{-10pt}
\end{figure*}
\begin{figure}[!h]
\vspace{-10pt}
\begin{center}{
		\centering
            \subfloat[Plant]{{\includegraphics[width=0.217\textwidth]{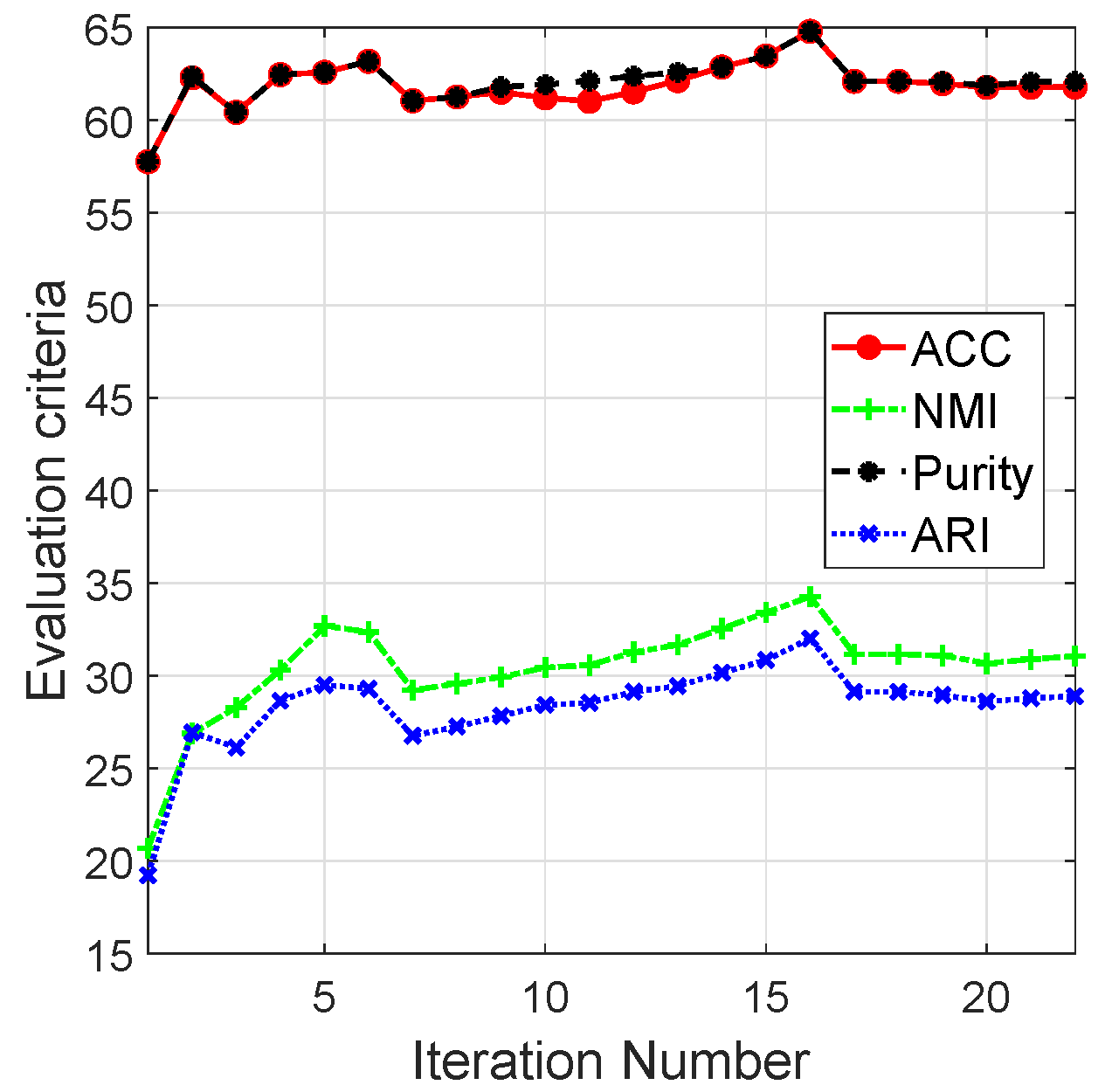}}}\hspace{1mm}
            \subfloat[CCV]{{\includegraphics[width=0.21\textwidth]{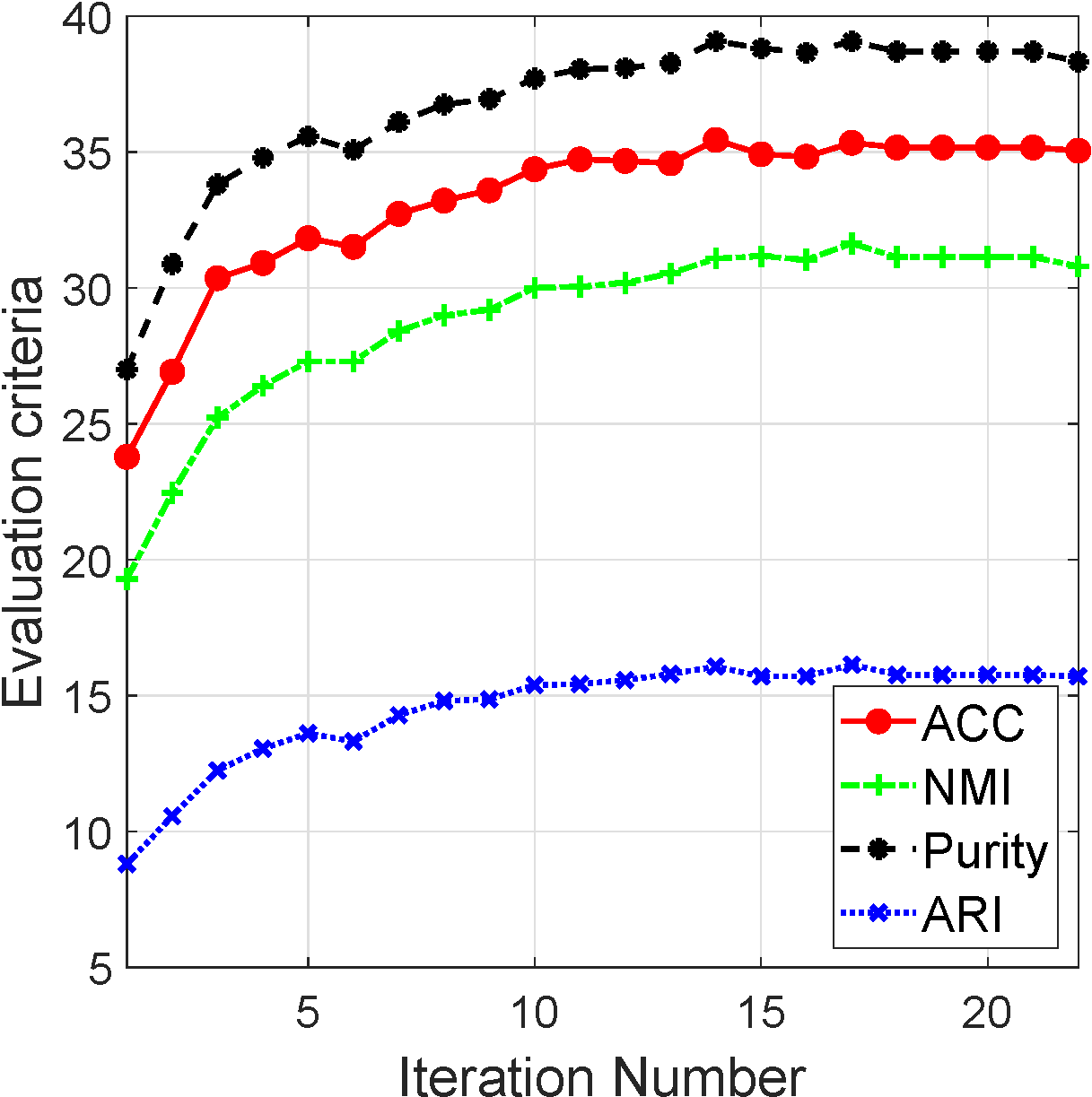}}}\\
            \vspace{-10pt}
            \subfloat[Flower102]{{\includegraphics[width=0.22\textwidth]{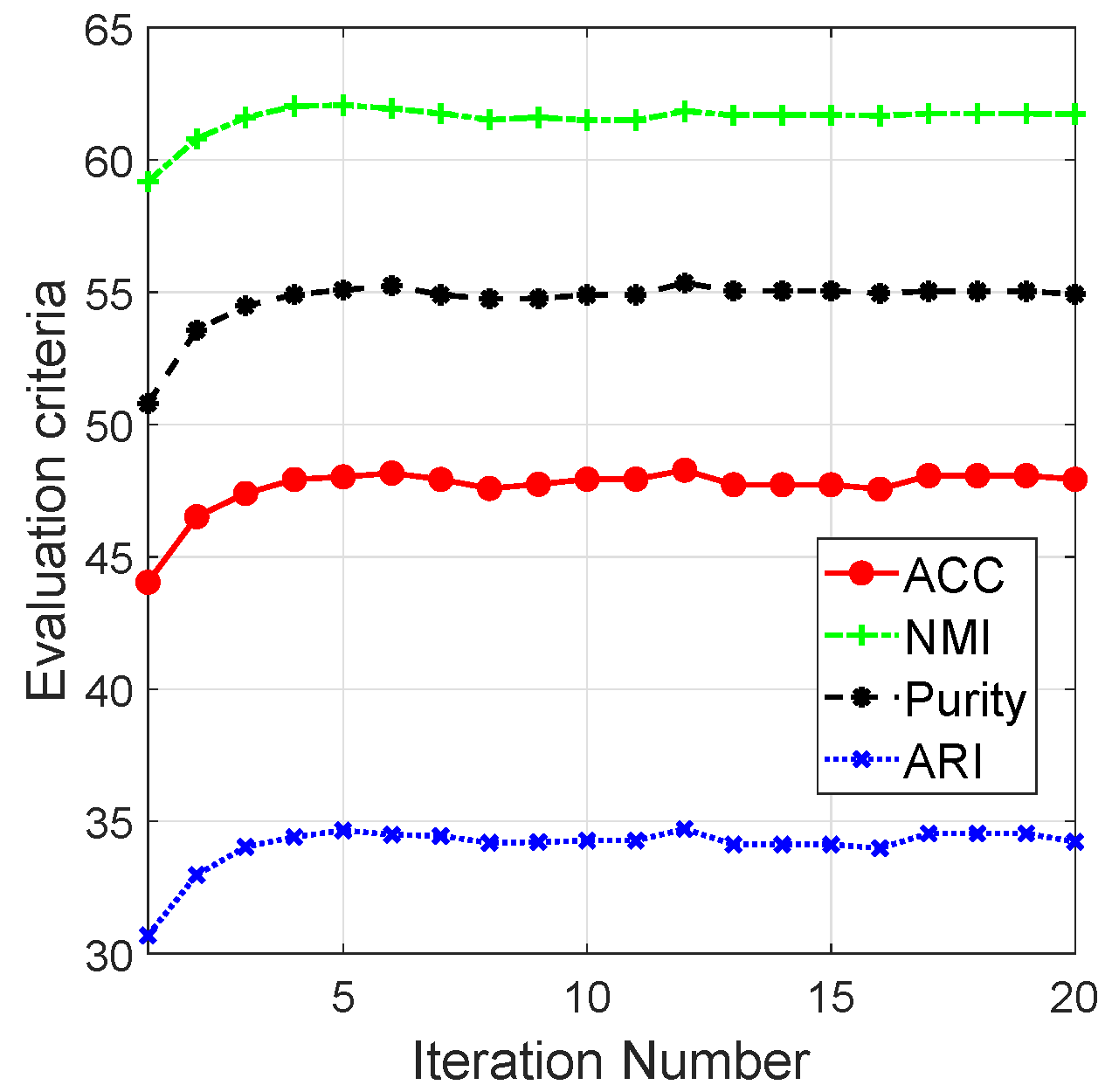}}}\hspace{1mm}
            \subfloat[Reuters]{{\includegraphics[width=0.22\textwidth]{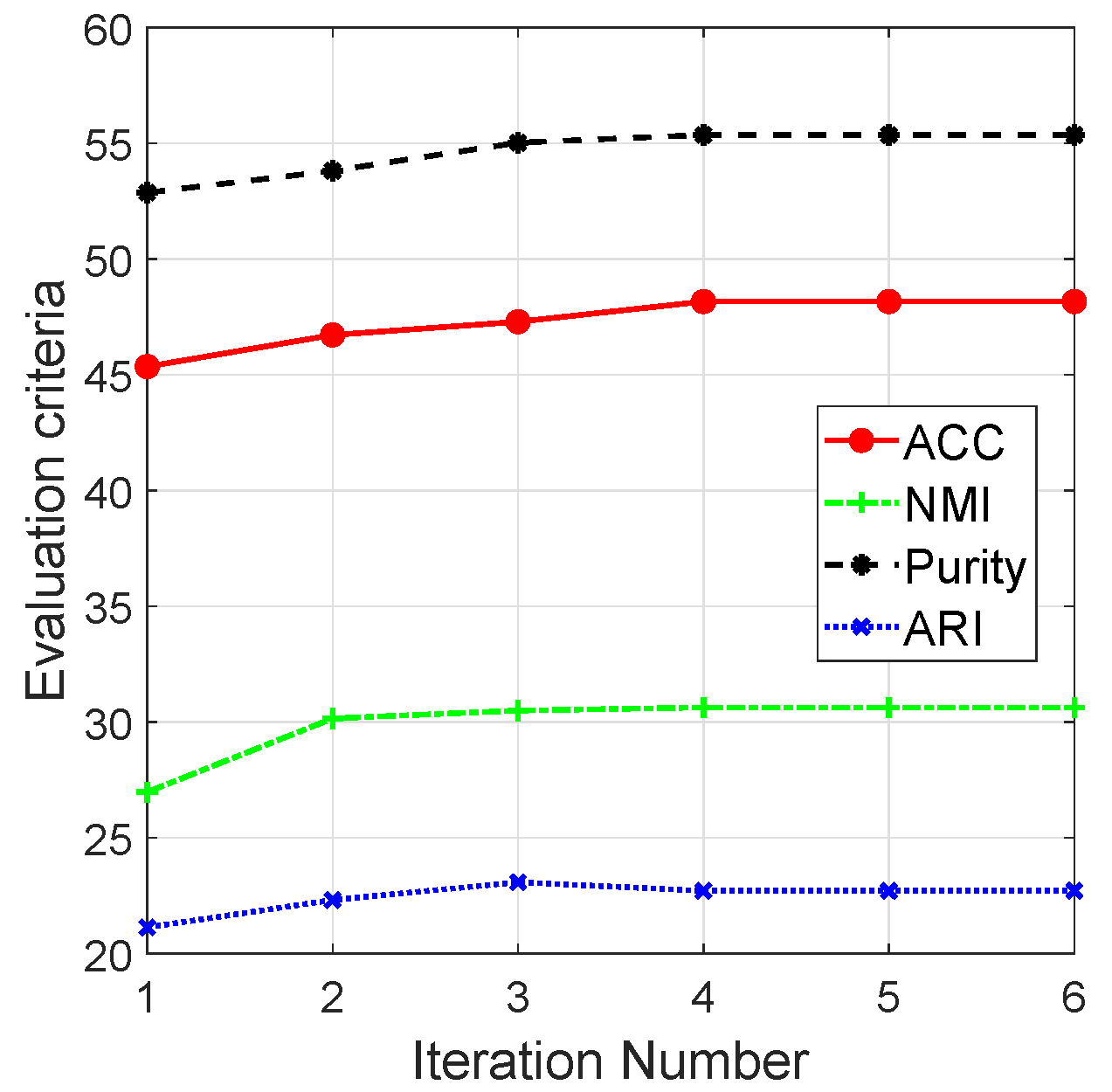}}}
			\caption{{Evolution of clustering performance during the iterations on Plant, CCV, Flower102, and Reuters.}}
            \label{evolution_criterion}
			}
\end{center}
\vspace{-15pt}
\end{figure}
\subsection{Computation Time Comparison}
Table \ref{Time_record} records the comparison of CPU time, including `Tune Time' and `Run Time'. `Tune Time' records the execution time including hyper-parameter tuning. `Run Time' records single running time. `-' denotes a parameter-free model. `N/A' denotes unavailable results due to memory-out or time-out errors. Note that OPLF is parameter-free but requires 20 times computation to reduce the randomness of $k$-means. From the results, we observe that:
\begin{enumerate}[1)]
\item Compared to late fusion based LFMKC, OPLF, and HMKC, our model achieves comparable or shorter `Run Time' than LFMKC and OPLF on most datasets, validating they share linear computational complexity. Not surprisingly, our model requires much less time than HMKC which has $\mathcal{O}(n^{3})$ complexity. Since the compared ones involve hyper-parameters or repeatably computation, they need much more `Tune Time' than our MKC-DNM.  
\item Compared to parameter-free kernel fusion based AKKM and MKKM with $\mathcal{O}(n^{3})$ complexity, although our model with $\mathcal{O}(n)$ complexity requires more time on several datasets, which is mainly due to our more complex optimization, we believe that more computational cost is a worthwhile sacrifice for much better clustering performance.
\item Compared to other methods with $\mathcal{O}(n^{3})$ complexity, our model exhibits significant superiority in effectiveness and efficiency. 
\end{enumerate}
\subsection{Comparison to Empirical HKMC Method}
To further compare our MKC-DNM algorithm and the strongest competitor HKMC method. Figure \ref{Compare_HKMC} plots the results of ours and HKMC on BBCSport, CCV and Flower102. 

Following the original hyper-parameter setting of \cite{liu2021hierarchical}, i.e. $c_{1}$ varies in [3k,\;4k,\;$\cdots$,\;10k] and $c_{2}$ varies in [2k,\;3k,\;$\cdots$,\;10k] with total 72 times computation by grid search. As the result shows, our parameter-free MKC-DNM surpasses HKMC with an obvious margin in the wide searching region. Although the proposed MKC-DNM may exhibit slight weaker performance at several searching regions, we emphasize that our model does not require time-consuming hyper-parameter tuning but achieves promising performance, which directly serves for unsupervised clustering.  
\subsection{Evolution of Clustering Performance}
To evaluate the effectiveness of our learning strategy to minimize dual noise within partition matrices in kernel space. Figure \ref{evolution_criterion} plots the evolution of ACC, NMI, purity, and ARI on Plant, CCV, Flower102, and Reuters. As can be seen, four metrics increase and keep stable during the iterations, which sufficiently illustrates the effectiveness of our algorithm.  
\section{Conclusion}\label{Se_Conclusion}
This work investigates an essential issue that how to minimize the noise inside the partition matrix. In this paper, we propose a novel parameter-free MKC algorithm with dual noise minimization to tackle this issue. Specifically, we first rigorously define the dual noise mathematically and separate it into dual parts, i.e. N-noise and C-noise, then we develop a unified and compact framework to minimize them. We design an efficient two-step iterative strategy to solve the model. To our best knowledge, it is the first time to investigate dual noise within the partition in the kernel space, distinguishing our work from existing researches. We observe that dual noise will pollute the block diagonal structures and incur the degeneration of clustering performance, and C-noise exhibits stronger destruction than N-noise. Minimizing dual noise significantly improves the clustering performance. Extensive experiments illustrate the effectiveness and efficiency of our proposed algorithm compared to the existing methods. 
\section{Acknowledgments}
This work was supported by the National Key R\&D Program of China 2020AAA0107100 and the National Natural Science Foundation of China (project no. 61773392, 61872377, 61922088 and 61976196).
\bibliographystyle{unsrt}
\bibliography{Reference}

\clearpage
\appendix
\section{Proof of Lemma }
Recall the definition of dual noise, i.e.
\begin{align}
&v\left(\mathbf{E}_p ^\mathrm{N}\right) \in N\left({\mathbf{H}^\ast}^\top\right) = \left\{ \mathbf{x} \;|\; {\mathbf{H}^\ast}^\top \mathbf{x} = \mathbf{0}\right\},\;\; \forall p \in \Delta_m, \label{eq:EN_appendix}\\
&v\left(\mathbf{E}_p ^\mathrm{C}\right) \in C\left(\mathbf{H}^\ast\right) = \left\{ \mathbf{y} \;|\; \exists \; \mathbf{x}, \; \mathrm{s.t.}\; \mathbf{y} = \mathbf{H}^\ast \mathbf{x}\right\},\;\; \forall p \in \Delta_m, \label{eq:EC_appendix}
\end{align}
where $\Delta_m = \{1,2,\cdots,m\}$. 
\subsection{Proof of Lemma 1}
\begin{lemma}\label{lemma1_appendix}
$\mathrm{Tr} \left( \mathbf{E}_p ^\mathrm{N} \mathbf{H}^\ast {\mathbf{H}^\ast}^\top \right) = 0$,\;\; $\forall p \in \Delta_m$.
\end{lemma}

\begin{proof}
The spectral decomposition of $\mathbf{E}_p ^\mathrm{N}$ is formulated as 
\begin{equation}\label{}
\begin{split}
\mathbf{E}_p ^\mathrm{N} = \sum_i \lambda_i v_i v_i ^\top.
\end{split}
\end{equation}
where $v_i$ is the unit eigenvector with corresponding eigenvalue $\lambda_i$.

According to Eq. (\ref{eq:EN_appendix}), we have $\lambda_i {\mathbf{H}^\ast}^\top v_i = 0$. We can further derive that:
\begin{equation}\label{}
\begin{split}
&\mathrm{Tr} \left( \mathbf{E}_p ^\mathrm{N} \mathbf{H}^\ast {\mathbf{H}^\ast}^\top \right) = \mathrm{Tr} \left( \sum_i \lambda_i v_i v_i ^\top \mathbf{H}^\ast {\mathbf{H}^\ast}^\top \right) \\
=& \mathrm{Tr} \left( \sum_i v_i ^\top \mathbf{H}^\ast \left(\lambda_i {\mathbf{H}^\ast}^\top v_i \right) \right) = 0.
\end{split}
\end{equation}
This completes the proof.
\end{proof}
\subsection{Proof of Lemma 2}
\begin{lemma}\label{lemma2_appendix}
$\mathrm{Tr} \left( \mathbf{E}_p ^\mathrm{C} \mathbf{H}^\ast {\mathbf{H}^\ast} ^\top \right) = \mathrm{Tr}\left( \mathbf{E}_p ^\mathrm{C}\right)$,\;\; $\forall p \in \Delta_m$.
\end{lemma}

\begin{proof}
The spectral decomposition of $\mathbf{E}_p ^\mathrm{C}$ is formulated as 
\begin{equation}\label{}
\begin{split}
\mathbf{E}_p ^\mathrm{C} = \sum_i \lambda_i v_i v_i ^\top.
\end{split}
\end{equation}
where $v_i$ is the unit eigenvector with corresponding eigenvalue $\lambda_i$.

According to Eq. (\ref{eq:EC_appendix}), we have $v_i = \mathbf{H}^\ast x_i$, where $x_i$ is a unit vector. We can further derive that: 
\begin{equation}\label{}
\begin{split}
&\mathrm{Tr} \left( \mathbf{E}_p ^\mathrm{C} \mathbf{H}^\ast {\mathbf{H}^\ast} ^\top \right) = \mathrm{Tr} \left( \sum_i \lambda_i v_i v_i ^\top \mathbf{H}^\ast {\mathbf{H}^\ast}^\top \right) \\
=& \mathrm{Tr} \left( \sum_i \lambda_i \left( {\mathbf{H} ^\ast}^\top v_i \right) ^\top \left( {\mathbf{H} ^\ast}^\top v_i \right) \right) \\
= &\mathrm{Tr} \left( \sum_i \lambda_i x_i ^\top x_i \right) = \sum_i \lambda_i = \mathrm{Tr}\left( \mathbf{E}_p ^\mathrm{C}\right).
\end{split}
\end{equation}
This completes the proof.
\end{proof}
\subsection{Proof of Lemma 3}
\begin{lemma}\label{lemma3_appendix}
$\mathbf{E}_p ^\mathrm{N}$ is positive semi-definite (PSD) and $\mathbf{E}_p ^\mathrm{C}$ is negative semi-definite (NSD).
\end{lemma}

\begin{proof}
For all $\mathbf{x} \in \mathbb{R}^n$, we separate it into two parts: 
\begin{equation}\label{}
\begin{split}
\mathbf{x} = \mathbf{x}_\mathrm{N} + \mathbf{x}_\mathrm{C},
\end{split}
\end{equation}
where $\mathbf{x}_\mathrm{N} \in N\left( {\mathbf{H} ^\ast}^\top \right)$ and $\mathbf{x}_\mathrm{C} \in C\left( {\mathbf{H} ^\ast} \right)$ are existing and unique. According to Eq. (\ref{eq:EC_appendix}), we have $\mathbf{x}_\mathrm{C} = \mathbf{H} ^\ast \mathbf{y}$, where $\|\mathbf{y}\|_2^2 = \|\mathbf{x}_\mathrm{C}\|_2^2$. Then we have 
\begin{equation}\label{}
\begin{split}
{\mathbf{H} ^\ast}^\top \mathbf{x}_\mathrm{N} = \mathbf{0}, \;\mathbf{E}_p ^\mathrm{C} \mathbf{x}_\mathrm{N} = \mathbf{0}, \;\mathbf{E}_p ^\mathrm{N} \mathbf{x}_\mathrm{C} = \mathbf{0}. 
\end{split}
\end{equation}
We can further derive that: 
\begin{equation}\label{}
\begin{split}
&\mathbf{x}^\top \mathbf{E}_p ^\mathrm{N} \mathbf{x} = \mathbf{x}_\mathrm{N} ^\top \mathbf{E}_p ^\mathrm{N} \mathbf{x}_\mathrm{N} \\
= &\mathbf{x}_\mathrm{N} ^\top \mathbf{U}_p \mathbf{U}_p^\top \mathbf{x}_\mathrm{N} - \mathbf{x}_\mathrm{N} ^\top \left( \mathbf{H}^\ast {\mathbf{H}^\ast}^\top + \mathbf{E}_p^{C} \right) \mathbf{x}_\mathrm{N}\\
= &\mathbf{x}_\mathrm{N} ^\top \mathbf{U}_p \mathbf{U}_p^\top \mathbf{x}_\mathrm{N} \geq 0, 
\end{split}
\end{equation}
and 
\begin{equation}\label{}
\begin{split}
&\mathbf{x}^\top \mathbf{E}_p ^\mathrm{C} \mathbf{x} = \mathbf{x}_\mathrm{C} ^\top \mathbf{E}_p ^\mathrm{C} \mathbf{x}_\mathrm{C} \\
= &\mathbf{x}_\mathrm{C} ^\top \mathbf{U}_p \mathbf{U}_p^\top \mathbf{x}_\mathrm{C} - \mathbf{x}_\mathrm{C} ^\top \left( \mathbf{H}^\ast {\mathbf{H}^\ast}^\top + \mathbf{E}_p^{N} \right) \mathbf{x}_\mathrm{C}\\
= &\mathbf{x}_\mathrm{C} ^\top \mathbf{U}_p \mathbf{U}_p^\top \mathbf{x}_\mathrm{C} - \mathbf{x}_\mathrm{C} ^\top \mathbf{H}^\ast {\mathbf{H}^\ast}^\top \mathbf{x}_\mathrm{C} \\
= &\mathbf{x}_\mathrm{C} ^\top \mathbf{U}_p \mathbf{U}_p^\top \mathbf{x}_\mathrm{C} - \mathbf{y}^\top \mathbf{y}\\
\leq &\mathbf{x}_\mathrm{C} ^\top \mathbf{x}_\mathrm{C} - \mathbf{y}^\top \mathbf{y} = 0. 
\end{split}
\end{equation}
This completes the proof. 
\end{proof}
\end{document}